%% file: 0Draft.tex
\documentclass[12pt]{article}
\usepackage[margin=1.1in]{geometry}
\input{preamble.tex}

\crefname{paragraph}{part}{parts}
\makeatletter
\def\@maketitle{%
  \newpage
  \begin{center}%
  \let \footnote \thanks
    {\Large \bf \@title \par}%
  \end{center}%
  \par
  \vskip 0.5em}
\makeatother

\title{Adaptive Norm-Based Regularization for Neural Networks}

\begin{document}
\maketitle

\begin{center}
{\large
\begin{tabular}{cc}
\makecell{
Muhammad Qasim$^{\dagger,*}$\\
{\small\texttt{muhammad.qasim@stat.lu.se}}
}
&
\makecell{
Farrukh Javed$^{\dagger}$\\
{\small\texttt{farrukh.javed@stat.lu.se}}
}
\end{tabular}
\\[0.8em]
\normalsize
\begin{tabular}{c}
$^{\dagger}$Department of Statistics, Lund University, Lund, Sweden\\[0.3em]
$^{*}${\small Corresponding author: \texttt{muhammad.qasim@stat.lu.se}}
\end{tabular}
}
\end{center}

\vspace{2em}

\begin{abstract}
In this paper, we study norm-based regularization methods for neural networks. We compare existing penalization approaches and introduce two regularization strategies that extend classical ridge- and lasso-type penalties to neural network models. The first strategy modifies weight decay by incorporating the covariance structure of the input features into a ridge-type $\ell_2$ penalty, allowing regularization to account for feature dependence. The second combines an $\ell_1$ sparsity penalty with covariance-aware $\ell_2$ regularization, producing neural network weights that are both sparse and structurally informed. Monte Carlo simulations are used to evaluate these methods under different data-generating settings, followed by two real-data applications on building cooling-load prediction and leukemia cell-type classification from high-dimensional gene expression data. Across simulated and real-data examples, the proposed regularizers improve predictive performance on unseen data and provide more effective complexity control than standard norm-based penalties, particularly when features are correlated or high-dimensional.
\end{abstract}

\setcounter{page}{1}
\newpage
\section{Introduction}

Regularization is an important principle in statistical learning in high-dimensional settings where overfitting is a major concern \citep{hoerl1970ridge,tibshirani1996regression}. In neural networks, which often contain millions of parameters, excessive model flexibility can cause the network to learn non-generalizable patterns from the training data, resulting in poor generalization on unseen data \citep{zhang2017understanding}. While collecting more data can help, it is often not feasible.  Regularization techniques, such as $\ell_2$ weight decay (Ridge, \cite{hoerl1970ridge}) and $\ell_1$ penalization (Lasso, \cite{tibshirani1996regression}) have been widely used in both statistics and machine learning to control model complexity and improve generalization \citep{goodfellow2016deep, zhang2023dive}. Figure \ref{fig:error_vs_complexity_curve} illustrates the effect of model complexity on predictive performance. The true relationship between the input and output is quadratic, but three polynomial regression models of increasing complexity were fitted. The linear model provides a misspecified fit to the curved data structure, yielding large errors on both training and testing sets, demonstrating underfitting. The quadratic model closely approximates the true function, with low and similar training and testing errors, representing an appropriate balance between bias and variance. On the other hand, an over-parameterized polynomial fits the noise in the training data, leading to overfitting.

\begin{figure}[ht]
    \centering
    \includegraphics[width=1.0\textwidth]{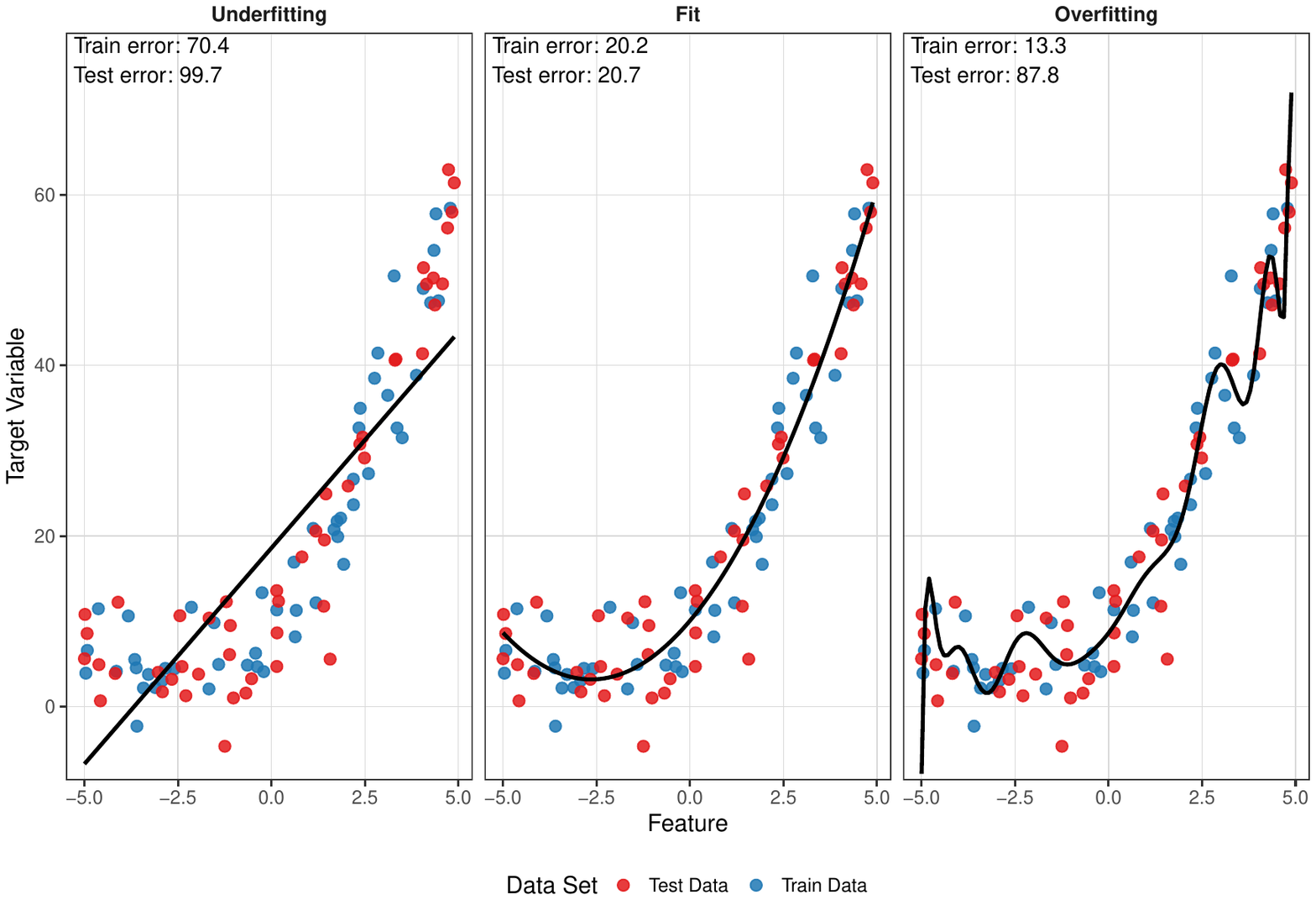} 
     \caption{\small Effect of model complexity on predictive performance using synthetic data. \textit{Left:} Fitted curves from linear (underfitting), quadratic (correct fit), and an overly complex polynomial (overfitting) models with training (blue) and test (red) points.
    }
    \label{fig:error_vs_complexity_curve}
\end{figure}

These challenges have motivated a broad literature on regularization methods for neural networks. In recent years, substantial effort has been devoted to improving the interpretability of neural networks and controlling overfitting in order to obtain more reliable predictions and classifications. Regularization is a standard tool for controlling overfitting across a wide range of deep learning architectures, including convolutional neural networks (CNNs) and models for sequential data (e.g., temporal and textual inputs). Accordingly, a variety of regularization techniques have been proposed and empirically assessed across model classes and application domains. \citet{santos2022avoiding} review regularization methods for CNNs and classify them into three broad categories, namely data augmentation, internal modifications, and label-based techniques. Data augmentation artificially expands the training set by applying transformations, noise injection, or perturbations to inputs. This strategy has proven effective in image, text \citep{shorten2021text}, and time-series tasks, as it encourages models to learn more robust, invariant features \citep{shorten2019survey}. Early stopping, which stops training when validation performance stops improving, is another widely used regularization technique. \cite{srivastava2014dropout} stated that the dropout method is an important regularization technique for different types of neural networks. Dropout has also been widely studied as a regularization strategy in neural networks, with empirical evidence suggesting that it can reduce overfitting and improve predictive performance. However, a commonly noted practical drawback is that dropout-trained models often require longer training times than standard networks with the same architecture. Moreover, dropout primarily acts as a training-time mechanism, whereas prediction is typically carried out using the full network \citep{ma2019transformed}.
Batch normalization \citep{ioffe2015batch} has been extensively investigated in deep learning, with numerous empirical studies reporting improved training stability and faster convergence, often allowing the use of larger learning rates and reducing sensitivity to initialization. Several works also note that, in some settings, batch normalization can reduce reliance on dropout. In a related direction, \citet{tomar2014manifold} propose a manifold learning-based regularization framework for deep neural networks and evaluate it on an automatic speech recognition speech-in-noise task, reporting word error rates in comparison with bottleneck networks trained without manifold constraints.

Several factors contributing to overfitting have been highlighted in the literature. For example, \citet{bejani2021systematic} discusses a range of sources, including heterogeneous pattern complexity that may call for different model capacities \citep{caruana2000overfitting}, biased or imbalanced training samples \citep{erhan2010does}, and the effect of non-negligible estimation variance \citep{cawley2010over}. Motivated by these observations, the simulation study considers a collection of settings designed to reflect such mechanisms and to assess how different regularized methods behave under representative overfitting scenarios.

Among the available approaches, norm-based regularization remains particularly useful because it directly controls model complexity through the training objective. Early work by \citet{krogh1991simple} reported that norm-based regularization, such as weight decay, can improve generalization in feed-forward neural networks, and subsequent studies have further documented its effectiveness for controlling weight magnitudes and improving predictive performance \citep{ishikawa1996structural, wu2006convergence}. Related contributions have explored adaptive variants and alternative penalties. For instance, \citet{bejani2019convolutional} develop an adaptively regularized CNN using adaptive dropout together with weight decay to mitigate overfitting in classification. \citet{bejani2020theory} propose an SVD-based stabilization approach in which a Tikhonov-type term is incorporated into the training objective to regularize weight updates toward an SVD approximation. \citet{ma2019transformed} introduce a non-convex transformed $\ell_1$ penalty to promote sparsity and remove non-informative connections in deep networks, while \citet{wu2014batch} study an $\ell_{1/2}$ penalty for connection pruning, motivated by its stronger sparsity-inducing behaviour relative to $\ell_1$. More recently, \citet{lemhadri2021lassonet} propose \textsc{LassoNet}, which encourages feature sparsity through a skip-connection mechanism that restricts feature participation in hidden layers unless the corresponding skip coefficient is active. Finally, \citet{zhang2019novel} consider combined $\ell_1$ and $\ell_2$ penalties for input-weight shrinkage in multilayer perceptrons.

These developments also connect closely to the broader literature on sparsity-inducing regularization. Sparse optimization--based approaches aim to simplify deep NNs by incorporating sparsity-inducing penalties during training, thereby encouraging many parameters to become exactly zero. In contrast to post hoc pruning procedures that remove weights after training, these methods learn a sparse structure directly through the optimization problem. A commonly noted advantage is that the resulting sparse network can be deployed at inference without requiring a separate pruning step, whereas dropout is primarily a training-time mechanism. Moreover, beyond model compression, several studies report that optimization-driven sparsity can, in some settings, match or even improve predictive performance relative to the corresponding dense networks \citep{alvarez2016learning, yoon2017combined}. Accordingly, $\ell_1$-norm regularization remains one of the most widely used convex surrogates for inducing sparsity, and it can be handled efficiently in many optimization settings to yield sparse solutions \citep{yu2014click, sun2016design, cui2016high, zhang2019feature, lemhadri2021lassonet}. However, in most existing work, the penalty structure is primarily first-order in nature and does not explicitly incorporate second-order information, such as covariance or Gram structure of the learned representations. There remains a need for regularization methods that penalize model complexity in a data-adaptive manner while maintaining or improving predictive accuracy, particularly in high-dimensional or noisy regimes.

\subsection{Contributions of this study}
The paper makes three main contributions.
\begin{itemize}

    \item The primary aim of this paper is to develop statistical foundations for geometry-aware regularization in deep learning by introducing and analyzing data-adaptive penalties for neural networks. Motivated by the need to align shrinkage with the second-order structure of learned representations, we propose two methods: \textit{Covridge}, which combines covariance-weighted quadratic shrinkage with  $\ell_2$ weight decay, and \textit{Sparridge} replaces the $\ell_2$ term with an $\ell_1$ penalty, combining sparsity with covariance-weighted shrinkage based on the empirical Gram structure.

    \item  The proposed methods address limitations of standard regularization in neural networks by introducing covariance-weighted, data-adaptive penalties. We provide theoretical results characterizing their large-sample behaviour and clarifying their impact on shrinkage geometry and model complexity.

    \item We conduct a Monte Carlo study that systematically compares norm-based regularization methods across varying sample sizes, dimensions, correlation structures, noise levels, and degrees of nonlinearity.

\end{itemize}

\subsection{Notation}

Here we introduce the notation used throughout the paper. Let $q \in [1, \infty)$ and let $\bm{x} = (x_1, \ldots, x_p)^\top \in \mathbb{R}^p$ denote a $p$-dimensional vector. The $\ell_q$ norm of $\bm{x}$ is defined as $\|\bm{x}\|_q \coloneqq \left(\sum_{j=1}^p |x_j|^q\right)^{1/q}$. The $\ell_\infty$ norm is given by $\|\bm{x}\|_{\infty} \coloneqq \max_{j \in [p]} |x_j|$, and the $\ell_0$ pseudo-norm, which counts the number of nonzero entries in $\bm{x}$, is defined as $\|\bm{x}\|_0 \coloneqq \sum_{j=1}^p \mathbbm{1}_{\{x_j \neq 0\}}$, where $\mathbbm{1}(\cdot)$ denotes the indicator function, which equals 1 if the condition holds and 0 otherwise.  
For a real-valued matrix $\bm{A} \in \mathbb{R}^{m \times n}$, the Frobenius norm is defined as $\|\bm{A}\|_F \coloneqq \left( \sum_{i=1}^m \sum_{j=1}^n a_{ij}^2 \right)^{1/2}$, which corresponds to the $\ell_2$ norm of $\bm{A}$ when viewed as a vector in $\mathbb{R}^{mn}$.  If $\bm{A}$ and $\bm{B}$ are two matrices, their inner product is defined as $\langle \bm{A}, \bm{B} \rangle \coloneqq \operatorname{tr}(\bm{A}^\top \bm{B})= \sum_{i=1}^m \sum_{j=1}^n a_{ij} b_{ij}$, where $\operatorname{tr}(\cdot)$ denotes the trace operator. Bold lowercase letters (e.g., $\bm{x}$, $\bm{w}$) denote vectors, and bold uppercase letters (e.g., $\bm{W}$, $\bm{A}$) denote matrices.
 

\section{Model Setup} \label{sec:model}

We consider a supervised learning problem where the observed data consist of input-target pairs $\{(\boldsymbol{x}_i, y_i)\}_{i=1}^n$, with $\boldsymbol{x}_i \in \mathbb{R}^p$ and $y_i \in \mathbb{R}$. Let $\boldsymbol{X} \in \mathbb{R}^{n \times p}$ denote the input matrix and $\boldsymbol{y} \in \mathbb{R}^n$ the response vector. The prediction function is modeled by a feedforward neural network $f(\cdot\,; \boldsymbol{\theta}) : \mathbb{R}^p \rightarrow \mathbb{R}$, parameterized by $\boldsymbol{\theta}$,  which collects all learnable weight matrices and offset vectors\footnote{The term \textit{bias} is commonly used in neural networks literature to refer to the offset vector added in each layer. However, this should not be confused with statistical bias, which refers to the difference between an estimator's expected value and the true parameter. To avoid confusion, we refer to these layer-specific vectors as offsets \citep{lindholm2022machine}.} $\boldsymbol{b},$ across the network. We solve the following optimization problem to find the optimal values for $\boldsymbol{\theta}$

\begin{equation}
\hat{\boldsymbol{\theta}} = \arg\min_{\boldsymbol{\theta}} J(\boldsymbol{\theta}),
\end{equation}

where $J(\boldsymbol{\theta}) = \frac{1}{n} \sum_{i=1}^n\mathcal{L}\left(y_i, f(\mathbf{x}_i; \boldsymbol{\theta})\right)$ is the cost function, $\mathcal{L}(\cdot, \cdot)$ is a pointwise loss, typically the squared error loss $\ell(y, \hat{y}) = (y - \hat{y})^2$ in regression problems and $f(\mathbf{x}_i; \boldsymbol{\theta})$ is the output of the neural network. Figure \ref{fig:deepnn} illustrates the structure of a feedforward neural network to provide a clear overview of the general deep neural framework.
\begin{figure}[htp]
\centering
\begin{tikzpicture}[scale=1, every node/.style={scale=1},
    input neuron/.style={
        circle, draw=green, fill=green!20,
        minimum size=0.9cm, inner sep=0pt
    },
    hidden neuron/.style={
        circle, draw=blue, fill=blue!20,
        minimum size=0.9cm, inner sep=0pt
    },
    output neuron/.style={
        circle, draw=orange, fill=orange!20,
        minimum size=0.9cm, inner sep=0pt
    },
    bias neuron/.style={
        circle, draw=gray, fill=gray!20,
        minimum size=0.9cm, inner sep=0pt
    },
    arrow style/.style={
        ->, >={latex}, line width=0.3mm, opacity=0.8
    }
]
\foreach \i in {1,2,3}
    \node[input neuron] (I\i) at (0,-\i) {$x_{\i}$};
\node at (0,-4.2) {$\vdots$};
\node[input neuron] (Ip) at (0,-5.4) {$x_p$};
\node[bias neuron] (I0) at (0,0) {$1$};

\foreach \i in {1,2,3}
    \node[hidden neuron] (H1\i) at (2.5,-\i) {$h$};
\node at (2.5,-4.2) {$\vdots$};
\node[hidden neuron] (H1U) at (2.5,-5.4) {$h$};
\node[bias neuron] (H1B) at (2.5,0) {$1$};

\foreach \i in {1,2,3}
    \node[hidden neuron] (H2\i) at (5,-\i) {$h$};
\node at (5,-4.2) {$\vdots$};
\node[hidden neuron] (H2U) at (5,-5.4) {$h$};
\node[bias neuron] (H2B) at (5,0) {$1$};


\node at (5.5,-3) {\rotatebox{90}{$\cdots$}};

\foreach \i in {1,2,3}
    \node[hidden neuron] (HLm1\i) at (7.5,-\i) {$h$};
\node at (7.5,-4.2) {$\vdots$};
\node[hidden neuron] (HLm1U) at (7.5,-5.4) {$h$};
\node[bias neuron] (HLm1B) at (7.5,0) {$1$};

\foreach \i in {1,2,3}
    \node[hidden neuron] (HL\i) at (10,-\i) {$h$};
\node at (10,-4.2) {$\vdots$};
\node[hidden neuron] (HLU) at (10,-5.4) {$h$};
\node[bias neuron] (HLB) at (10,0) {$1$};

\node[output neuron] (O) at (13,-2.5) {$\hat{y}$};

\foreach \i in {0,1,2,3,p}
    \foreach \j in {1,2,3,U}
        \draw[arrow style] (I\i) -- (H1\j);

\foreach \i in {1,2,3,U,B}
    \foreach \j in {1,2,3,U}
        \draw[arrow style] (H1\i) -- (H2\j);

\foreach \i in {1,2,3,U,B}
    \foreach \j in {1,2,3,U}
        \draw[arrow style, dashed] (H2\i) -- (HLm1\j);

\foreach \i in {1,2,3,U,B}
    \foreach \j in {1,2,3,U}
        \draw[arrow style] (HLm1\i) -- (HL\j);

\foreach \i in {1,2,3,U, B}
    \draw[arrow style] (HL\i) -- (O);

\node[align=center] at (0,1) {Input\\variables};
\node[align=center] at (2.5,1) {Hidden\\Layer 1};
\node[align=center] at (4.7,1) {Hidden\\Layer 2};
\node[align=center] at (6.2,1) {{$\cdots$}};
\node[align=center] at (7.8,1) {Hidden\\Layer $L{-}1$};
\node[align=center] at (10,1) {Hidden\\Layer $L$};
\node[align=center] at (13,1) {Output};

\node at (1.3,-7.3) {\shortstack{Layer 1\\$(\boldsymbol{W}^{(1)},\ \boldsymbol{b}^{(1)})$}};
\node at (3.9,-7.3) {\shortstack{Layer 2\\$(\boldsymbol{W}^{(2)},\ \boldsymbol{b}^{(2)})$}};
\node at (6.2,-7.3) {\shortstack{$\cdots$}};
\node at (9,-7.3) {\shortstack{Layer $L{-}1$\\$(\boldsymbol{W}^{(L-1)},\ \boldsymbol{b}^{(L-1)})$}};
\node at (12.3,-7.3) {\shortstack{Layer $L$\\$(\boldsymbol{W}^{(L)},\ \boldsymbol{b}^{(L)})$}};

\end{tikzpicture}
\caption{An illustration of a deep neural network with multiple hidden layers.}
\label{fig:deepnn}
\end{figure}
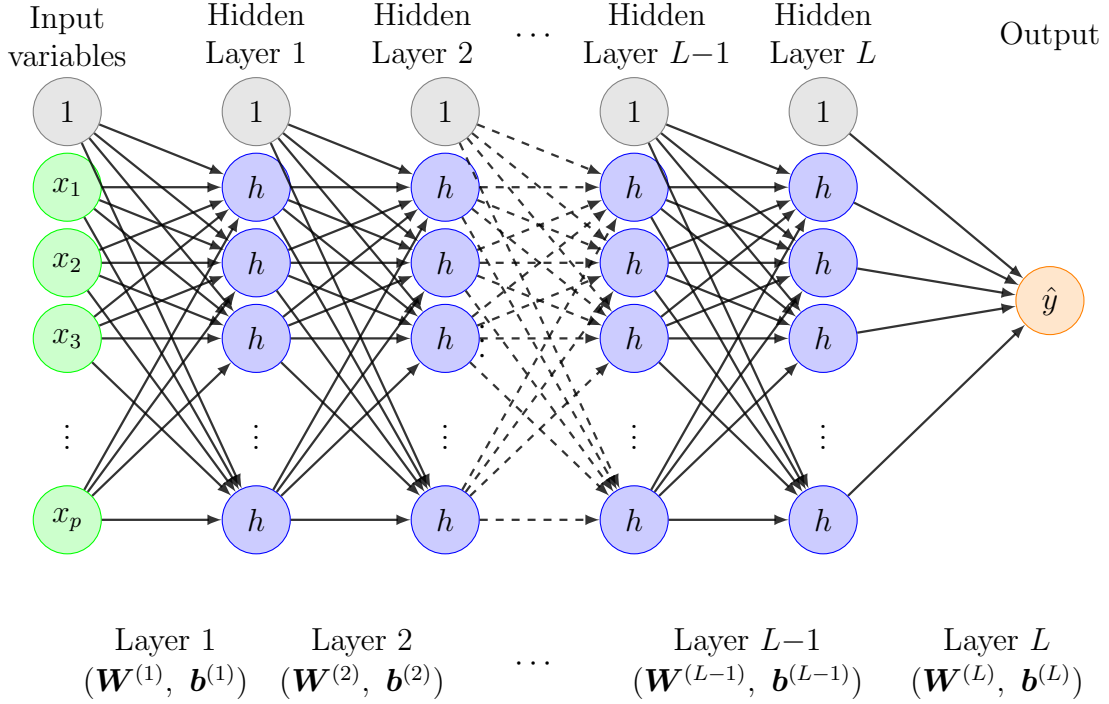


We follow similar notations to illustrate the neural network, as described in \cite{lindholm2022machine}, with omitted hidden layers represented by ``$\cdots$'' and a dashed arrow between Layer 2 and Layer \( L-1 \). To simplify the understanding of the neural network architecture and its parameters, we split the overall network parameters as 
$$
\boldsymbol{\theta} = (\boldsymbol{W}^{(1)}, \boldsymbol{b}^{(1)}, \boldsymbol{\theta}^*),$$
where $\boldsymbol{W}^{(1)} \in \mathbb{R}^{r_1 \times p}$ and $\boldsymbol{b}^{(1)} \in \mathbb{R}^{r_1}$ denote the weight matrix and offset vector of the first layer, respectively, and $\boldsymbol{\theta}^*$ collects the parameters of the remaining layers. For a generic input $\boldsymbol{x} \in \mathbb{R}^p$, the first-layer transformation is
$$
\boldsymbol{q}^{(1)} = h\left( \boldsymbol{W}^{(1)} \boldsymbol{x}+ \boldsymbol{b}^{(1)} \right), $$
where  $h(\cdot)$ represents the activation function, applied componentwise. In this paper, we take $h$ to be the ReLU activation function, defined by $h(a) = \max(0, a)$.  Similarly, the calculations for the remaining layers follow the same structure. For a detailed mathematical representation of deep neural networks and the backpropagation algorithm, we refer the reader to \cite{petersen2024mathematical} and \cite{lindholm2022machine}.

\subsection{Regularized objective function}

Regularization methods are commonly used to control the complexity of different models, such as neural networks, linear regression, or logistic regression. This is done by adding a penalty term \( \Omega(\boldsymbol{\theta}) \) to the objective function \( J(\boldsymbol{\theta}) \), which limits the model's capacity. The regularized objective function is given by

\begin{equation}
\widetilde{J}(\boldsymbol{\theta}; \boldsymbol{X}, \boldsymbol{y}) = J(\boldsymbol{\theta}) + \lambda \Omega(\boldsymbol{\theta}),
\label{eq:regobjfun}
\end{equation}

where $\lambda \in \mathbb{R}_{\geq 0}$ is the regularization parameter that controls how much the penalty term $\Omega(\boldsymbol{\theta})$ contributes to the total objective function, relative to the original objective function $J(\boldsymbol{\theta})$ without regularization. In neural networks, the penalty term typically only affects the weights of the affine transformations in each layer, while the bias vectors are left unregularized during the training. This is because weights determine the interaction between variables and generally require more data to estimate accurately, whereas each bias affects only a single variable and is therefore less prone to overfitting. So, excluding biases from regularization usually introduces little additional variance, while regularizing them can introduce a significant amount of underfitting \citep{goodfellow2016deep}. When $ \lambda = 0 $, the objective function in  \eqref{eq:regobjfun} becomes the unregularized objective function. 

We omit the offset terms in the model for simplicity, and use \(\boldsymbol{W}\) to denote the collection of trainable weight matrices. Under this convention, the regularized objective can be written as
$$
\widetilde{J}(\boldsymbol{W}; \boldsymbol{X}, \boldsymbol{y})=
J(\boldsymbol{W})+\lambda \Omega(\boldsymbol{W}).
$$
The gradient of the regularized objective with respect to the weights is then given by
\begin{equation}
\nabla_{\boldsymbol{W}} \widetilde{J}(\boldsymbol{W}; \boldsymbol{X}, \boldsymbol{y}) = \nabla_{\boldsymbol{W}} J(\boldsymbol{W}) + \lambda \nabla_{\boldsymbol{W}} \Omega(\boldsymbol{W}).
\label{eq:reggradient}
\end{equation}

At each iteration, the weights are updated using a single gradient descent step as follows
\begin{equation}
\boldsymbol{W} \leftarrow \boldsymbol{W} - \eta \nabla_{\boldsymbol{W}} \widetilde{J}(\boldsymbol{W}; \boldsymbol{X}, \boldsymbol{y}),
\label{eq:regupdate}
\end{equation}

where $\eta>0$ is the learning rate, which controls the size of each update and is also referred to as the step size. The parameters of neural networks are typically estimated through stochastic gradient descent (SGD), a widely used optimization method in which the gradient is evaluated iteratively during training. In practice, the gradient is often computed from mini-batches, or subsamples of the training data, rather than from the full dataset. This mini-batch formulation introduces stochasticity into the optimization process, which substantially improves computational efficiency and may also support convergence in large-scale learning problems; see, for example, \cite{bottou1991stochastic, bottou2012stochastic}.

In regularized neural networks, the optimization problem also depends on the tuning parameter $\lambda$, which controls the strength of regularization and therefore directly affects predictive performance. Several approaches for selecting $\lambda$ have been discussed in the literature, with cross-validation being one of the most common data-driven choices because it targets out-of-sample prediction accuracy. Among these approaches, common variants include $K$-fold and leave-one-out cross-validation. Beyond its role as a tuning parameter selected for predictive performance, $\lambda$ also acts as a shrinkage parameter that determines how strongly the weight matrices $\boldsymbol{W}$ are regularized through the $\ell_1$ or $\ell_2$ norm. This perspective naturally leads to the norm-based regularization methods considered in the following sections.

\section{Norm-Based Regularization  in Neural Networks} \label{sec:normBasedRegularization}
This section presents a regularization framework for neural networks. We begin by reviewing standard penalties used to control complexity and mitigate overfitting, and we then introduce geometry-aware alternatives motivated by the covariance structure of learned representations.

\subsection{Standard  methods}
Penalization based on $\ell_2$ and $\ell_1$ norms is widely used in neural networks to control model complexity and reduce overfitting. A combination of these norms, commonly referred to as Elastic Net regularization, has also been used in the neural-network literature; see, for example, \citet{chen2018ead, zhang2019novel}. The $\ell_2$ penalty is often referred to as weight decay because it adds a term proportional to the current weights to the gradient, so shrinking the weights toward zero at each update. In the classical regression literature, it is known as \emph{ridge regression} (or \emph{Tikhonov regularization}), introduced by \citet{hoerl1970ridge}. The corresponding penalty is
$$
\Omega_{\text{ridge}}(\mathbf{W}) = \|\mathbf{W}\|_F^2 =  \sum_{i,j} w_{ij}^2.
$$
This penalty yields smooth shrinkage of the weights towards zero as $\lambda$ increases, without inducing exact zeros. It can also be interpreted as the negative log-density of a Gaussian prior on $\mathbf W$, and it is well known to stabilise estimation in the presence of collinearity. Adding this penalty to $J(\mathbf W)$ modifies a gradient step as
$$
\mathbf{W} \leftarrow \mathbf{W} - \eta\bigl(\nabla_{\mathbf{W}} J(\mathbf{W}) + 2\lambda \mathbf{W}\bigr)
= (1-2\eta\lambda)\mathbf W - \eta \nabla_{\mathbf{W}}J(\mathbf W).
$$
While ridge-type shrinkage improves numerical stability and generalization, it does not, by itself, promote sparsity, since coefficients are continuously shrunk rather than thresholded.

Penalizing the $\ell_1$ norm, on the other hand, provides a direct and effective mechanism for inducing sparse weight matrices, which is particularly valuable in large networks where interpretability and computational efficiency are central considerations. The resulting Lasso  penalty (\citet{tibshirani1996regression}) becomes, 
$$
\Omega_{\text{lasso}}(\mathbf{W}) =  \left\|\mathbf{W}\right\|_1 = \sum_{i,j} |w_{ij}|.
$$
Unlike the $\ell_2$ penalty, the $\ell_1$ penalty is not differentiable at zero, and its geometry (a sharp corner at the origin) leads to thresholding behaviour that can produce exact zeros; see \citet{hastie2009elements}.  The subgradient can be written as,
\[
\partial_{\mathbf W}\Omega_{\text{lasso}}(\mathbf W)
= \mathrm{sign}(\mathbf W),
\]
where $\mathrm{sign}(\cdot)$ is applied elementwise and, at zero, the subgradient takes values in the interval $[-\lambda,\lambda]$ entrywise. A generic subgradient step can be written as
\[
\mathbf{W} \leftarrow \mathbf{W} - \eta\bigl(\nabla_{\mathbf{W}} J(\mathbf{W}) + \lambda \bm G\bigr),
\qquad \bm G\in \partial_{\mathbf W}\Omega_{\text{lasso}}(\mathbf W).
\]

However, many applications require both sparsity and numerical stability. The $\ell_1$ regularization can become unstable in the presence of strong collinearity, whereas  $\ell_2$ regularization yields stable shrinkage but does not perform variable selection. Elastic Net combines these two mechanisms by augmenting the $\ell_1$ penalty with an additional $\ell_2$ term that stabilizes the fit \citep{zou2005regularization}. In its standard form,
$$
\Omega_{\text{en}}(\mathbf{W}) =\alpha \|\boldsymbol{W}\|_1
+
\frac{1-\alpha}{2}\|\boldsymbol{W}\|_F^2, = \alpha \sum_{i,j}|w_{ij}|  + \frac{1-\alpha}{2}\sum_{i,j}w_{ij}^2,
\qquad \alpha \in [0,1].
$$
The $\ell_1$ component encourages sparsity, while the $\ell_2$ component provides additional shrinkage and improves stability when learned features are strongly correlated. A gradient/subgradient step for Elastic Net is therefore
$$
\mathbf{W} \leftarrow \mathbf{W} - \eta\Bigl(\nabla_{\mathbf{W}} J(\mathbf{W}) + \lambda \alpha \operatorname{sign}(\boldsymbol{W})
+\lambda (1-\alpha)\boldsymbol{W}\Bigr),
$$
with the understanding that $\mathrm{sign}(\mathbf W)$ is interpreted as a subgradient at entries equal to zero. This viewpoint will be useful below, since our proposed penalties retain the combination of a quadratic stabilization term and a sparsity-inducing $\ell_1$ term, while replacing the isotropic $\ell_2$ component by a covariance-weighted quadratic penalty that adapts the shrinkage geometry to the representation.


\subsection{Geometry-aware regularization weights}

Building on this perspective, we now introduce geometry-aware penalties that retain the stabilizing role of a quadratic term and the sparsity-inducing role of an $\ell_1$ term, while allowing the strength of shrinkage to adapt to the second-order structure of the learned representations. Neural networks are widely used in applications such as image recognition, natural language processing, audio processing, and sequential data modelling, where deep architectures with many layers and a large number of parameters are often required. This flexibility makes training challenging and increases the risk of overfitting \citep{goodfellow2016deep}. Classical regularization methods shrink weights during training, yet in their standard forms, they do not explicitly account for the second-order structure of the representations used by the model. In particular, weight decay applies uniform shrinkage across all parameter directions, whereas $\ell_1$ regularization promotes sparsity through coordinatewise thresholding, without adapting the strength of shrinkage to correlations in the data. Motivated by these limitations, we propose two geometry-aware regularizers that incorporate a data-driven covariance structure to induce adaptive, data-informed shrinkage. The primary objective is to employ regularization strategies that better reflect the complexity of the model and more effectively control the structure induced by dependence patterns in the training representations.

Let $\mathbf{H}\in\mathbb{R}^{n\times p}$ denote a user-chosen representation matrix (e.g., hidden-layer activations, transformed features, or a structured design), and define the empirical Gram matrix

$$
\mathbf{C}_n=\frac{1}{n}\mathbf{H}^\top\mathbf{H},
\qquad 
\mathbf{C}_{\delta,n}=\mathbf{C}_n+\delta\mathbf{I}_p,\ \delta>0.
$$
The choice of $\mathbf H$ is flexible and may encode prior structure, feature maps, or dimension-reduction transforms.\footnote{For example, $\mathbf H$ may represent a subset of features, principal components, kernel embeddings, or other pre-specified design structures.}

\subsubsection{Covridge}
The $\ell_2$ regularization is widely used in deep learning to control model complexity and mitigate overfitting \citep{krogh1991simple, ong2016dynamically, van2017l2, nakamura2019adaptive, lu2020low}. We propose a modified strategy, \emph{Covridge}\footnote{The method incorporates a data-driven Gram matrix into a ridge-type penalty, thereby weighting shrinkage according to the second-order structure of the representation.}, which combines covariance-weighted quadratic shrinkage with standard $\ell_2$ weight decay. Given a loss function $J(\boldsymbol{W})$, we define the regularized objective function
$$
J(\boldsymbol{W}) + \Omega_{\text{covridge}}(\mathbf{W}),
$$
where the Covridge penalty is
\begin{equation}
\Omega_{\text{covridge}}(\mathbf{W})
= \lambda_1 \left\| (\mathbf{C}_{\delta,n})^{1/2}\mathbf{W} \right\|_F^2
+ \lambda_2 \|\mathbf{W}\|_F^2,
\label{eq:covridge}
\end{equation}

with $\lambda_1\ge 0$ controls the strength of the covariance-weighted quadratic term and $\lambda_2\ge 0$ controls the isotropic $\ell_2$ component. The matrix $\mathbf{C}_{\delta,n}$ weights shrinkage directions according to dependence patterns in the representation matrix $\mathbf H$, whereas the term $\lambda_2\|\mathbf W\|_F^2$ applies uniform penalization across all weights. When $\lambda_1=0$, \eqref{eq:covridge} reduces to standard $\ell_2$ weight decay. In the eigenbasis of $\mathbf C_n$, the penalty assigns larger quadratic weights to directions associated with larger eigenvalues, thereby inducing anisotropic shrinkage aligned with the spectral geometry of the representation.

\subsubsection{Sparridge}

In large neural networks, sparsity is often desirable both for interpretability and for computational efficiency, since zero-valued weights can be removed from storage and computation \citep{mishra2021accelerating, kepner2018sparse, ma2019transformed, kepner2019sparse, zhu2021bearing}. While Covridge induces geometry-aware shrinkage, it does not, in general, produce exact zeros. To couple covariance-weighted shrinkage with sparsity, we propose \emph{Sparridge}, defined by the regularized objective
$$
J(\boldsymbol{W}) + \Omega_{\text{sparridge}}(\mathbf{W}),
$$
where
\begin{equation}
\Omega_{\text{sparridge}}(\mathbf{W})
= \lambda_1 \left\| (\mathbf{C}_{\delta,n})^{1/2}\mathbf{W} \right\|_F^2
+ \gamma \|\mathbf{W}\|_1,
\label{eq:sparridge}
\end{equation}
with $\gamma\ge 0$ controls the strength of the sparsity-inducing $\ell_1$ term. The $\ell_1$ component promotes sparse solutions through coordinatewise thresholding, whereas the covariance-weighted quadratic term controls the overall scale and directional variation of the weights in a manner aligned with the dependence structure encoded by $\mathbf{C}_{\delta,n}$. When $\lambda_1=0$, \eqref{eq:sparridge} reduces to standard $\ell_1$ regularization. Moreover, the quadratic form $\mathrm{tr}(\mathbf W^\top \mathbf C_{\delta,n}\mathbf W)=\| \mathbf C_{\delta,n}^{1/2}\mathbf W\|_F^2$ can be interpreted as a Mahalanobis-type penalty induced by the representation geometry. In practice, both Covridge and Sparridge can be applied to any dense layer by constructing $\mathbf C_{\delta,n}^{1/2}$ from hidden representations of the training data, for example, from $\mathbf H_{\text{train}}$ obtained via a forward pass at the corresponding layer.

Although Sparridge and Elastic Net both combine an $\ell_1$ term with a quadratic stabilization component, their quadratic parts differ in a fundamental way. In Elastic Net, the term $\|\mathbf W\|_F^2$ is isotropic and therefore applies the same amount of shrinkage in every parameter direction, irrespective of any dependence structure in the representation. Sparridge replaces this isotropic quadratic penalty by the covariance-weighted form $\lambda_1\,\mathrm{tr}\!\bigl(\mathbf W^\top \mathbf C_{\delta,n}\mathbf W\bigr)$, which induces direction-dependent shrinkage. In particular, in the eigenbasis of $\mathbf C_{\delta,n}$, directions associated with larger eigenvalues receive stronger quadratic penalization. Consequently, Sparridge adapts regularization to the second-order structure of the representation and provides a geometry-aware alternative to isotropic $\ell_2$ shrinkage.


\section{Geometric Interpretation of the Proposed Penalties} 

The asymptotic sampling distributions derived for the proposed estimators describe their large-sample fluctuations, but they do not reveal by themselves how the underlying penalties reshape the geometry of the parameter space. We therefore complement the limit theory with a geometric interpretation that highlights the role of the covariance-weighted quadratic term. This perspective also clarifies how Covridge and Sparridge differ from classical penalties such as Ridge, Lasso, and Elastic Net. Ridge imposes isotropic shrinkage that is agnostic to the geometry of the data, Lasso promotes sparsity through coordinatewise thresholding without directional adaptation, and Elastic Net combines these two effects while retaining an isotropic $\ell_2$ component. By contrast, Covridge and Sparridge explicitly incorporate the spectral structure of a data-driven Gram matrix, yielding direction-sensitive shrinkage aligned with dominant modes of variability in the representation. This eigenvalue-weighted viewpoint motivates the contour illustrations presented below and provides intuition for why the proposed penalties can be advantageous in multicollinear or highly anisotropic designs.

For concreteness, consider a linear map parameterised by a weight matrix $\mathbf W\in\mathbb R^{p\times d}$ and an empirical Gram matrix $\mathbf C_n\in\mathbb R^{p\times p}$, with the stabilized version $\mathbf C_{\delta,n}=\mathbf C_n+\delta\mathbf I_p$ for some $\delta>0$. The Covridge penalty can be written as
$$
\Omega_{\mathrm{covridge}}(\mathbf W)
= \lambda_1\|\mathbf C_{\delta,n}^{1/2}\mathbf W\|_F^2+\lambda_2\|\mathbf W\|_F^2
= \mathrm{tr}\!\left(\mathbf W^\top(\lambda_1\mathbf C_{\delta,n}+\lambda_2\mathbf I_p)\mathbf W\right).
$$
Let $\mathbf C_{\delta,n}=\mathbf U_n\boldsymbol\Lambda_{\delta,n}\mathbf U_n^\top$ be the spectral decomposition, where
$\boldsymbol\Lambda_{\delta,n}=\mathrm{diag}(\mu_{1n}+\delta,\dots,\mu_{pn}+\delta)$, and define eigen-coordinates $\widetilde{\mathbf W}=\mathbf U_n^\top \mathbf W$. Since $\mathbf U_n$ is orthogonal, the Frobenius norm is invariant under this rotation, and therefore
$$
\bigl\|\mathbf{C}_{\delta,n}^{1/2}\mathbf W\bigr\|_F^2
=\bigl\|\boldsymbol\Lambda_{\delta,n}^{1/2}\widetilde{\mathbf W}\bigr\|_F^2
=\sum_{i=1}^p(\mu_{in}+\delta)\,\|\widetilde{\mathbf w}_i\|_2^2,
\qquad
\|\mathbf W\|_F^2=\|\widetilde{\mathbf W}\|_F^2=\sum_{i=1}^p\|\widetilde{\mathbf w}_i\|_2^2,
$$
where $\widetilde{\mathbf w}_i^\top$ denotes the $i$th row of $\widetilde{\mathbf W}$. Substituting into the penalty yields
$$
\Omega_{\mathrm{covridge}}(\mathbf W)
= \lambda_1\|\boldsymbol\Lambda_{\delta,n}^{1/2}\widetilde{\mathbf W}\|_F^2+\lambda_2\|\widetilde{\mathbf W}\|_F^2
= \sum_{i=1}^p\bigl(\lambda_1(\mu_{in}+\delta)+\lambda_2\bigr)\,\|\widetilde{\mathbf w}_i\|_2^2.
$$

Through this decomposition, it can be seen that the quadratic component of Covridge induces direction-dependent shrinkage, so directions corresponding to larger eigenvalues of $\mathbf C_n$ are penalized more heavily. In other words, the eigenvalue weighting acts as a form of spectral filtering that discourages excessive parameter growth along dominant modes of variability in the representation, while imposing comparatively weaker shrinkage along lower-variance directions. In contrast to Ridge, which shrinks all directions uniformly, Covridge adapts the strength of quadratic penalization to the data geometry through the spectrum of $\mathbf C_n$. Moreover, unlike Lasso, it achieves directional control through continuous shrinkage rather than coordinatewise thresholding.


Similarly, the Sparridge builds on the same covariance-weighted quadratic term with an $\ell_1$ penalty, thereby combining direction-sensitive shrinkage with sparsity in the original parameterization. Relative to Elastic Net, Sparridge retains the sparsity-inducing mechanism of the $\ell_1$ penalty but replaces isotropic $\ell_2$ shrinkage by eigenvalue-weighted quadratic shrinkage, thereby aligning regularization with the empirical geometry of the representation.

\begin{figure}[htb]
\centering
\includegraphics[width=1.02\textwidth]{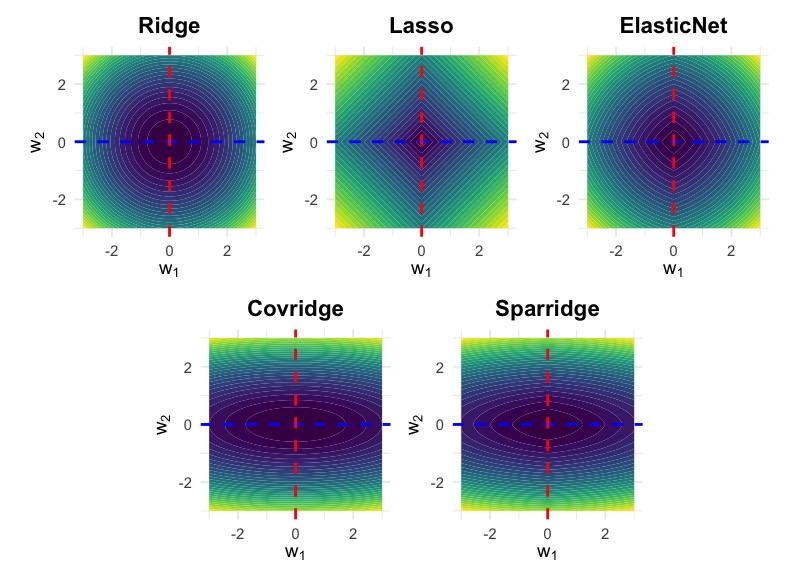}
\caption{Penalty landscapes in a two-dimensional eigenspace. The horizontal axis $w_1$ corresponds to a low-variance eigen-direction of the Gram matrix $\mathbf C_n$, while the vertical axis $w_2$ corresponds to a high-variance eigen-direction. Ridge is isotropic, whereas Covridge and Sparridge exhibit eigenvalue-weighted anisotropic shrinkage, producing contours that are compressed along the high-variance direction. The $\ell_1$ component in Lasso, Elastic Net, and Sparridge introduces corners that promote sparsity.}
\label{fig:contours-penalties}
\end{figure}

To visualise these geometric effects and connect them to familiar penalty shapes, Figure~\ref{fig:contours-penalties} presents contour plots of several regularizers in a two-dimensional eigenspace. Here $(w_1,w_2)$ are eigen-coordinates of $\mathbf C_n$, with $w_1$ aligned with a low-variance eigen-direction and $w_2$ aligned with a high-variance eigen-direction. Ridge produces isotropic circular contours, Lasso yields diamond-shaped contours with corners on the coordinate axes, and Elastic Net interpolates between these geometries. By contrast, Covridge generates anisotropic ellipses that are compressed along the high-variance direction, while Sparridge combines this anisotropy with the corner structure induced by the $\ell_1$ component.

\section{Properties of Covridge and Sparridge Estimators}
\label{sec:theoretical_results}
We study the large-sample behaviour of the proposed Covridge and Sparridge estimators. Both procedures are penalised least-squares estimators that augment the standard quadratic loss with a penalty involving the empirical covariance structure of a secondary design matrix. For each $n$, consider the fixed-design linear model
 $$y_{i,n} = \bm{x}_{i,n}^\top \bm{w}_0 + \varepsilon_{i,n},  i = 1,\dots,n,$$ 
 
 where $\bm x_{i,n}\in\mathbb R^p$ is the $i$th row of the (nonrandom) design matrix $\mathbf X_n\in\mathbb R^{n\times p}$, $\bm w_0\in\mathbb R^p$ is the target parameter vector, and $\{\varepsilon_{i,n}\}_{i=1}^n$ are i.i.d.\ errors with $\mathbb E[\varepsilon_{i,n}]=0$. Let $\bm y_n=(y_{1,n},\dots,y_{n,n})^\top$ and $\bm\varepsilon_n=(\varepsilon_{1,n},\dots,\varepsilon_{n,n})^\top$. We impose the following conditions to obtain the asymptotic distributions.

\begin{assump}[label=(A\arabic*)]
\item \label{ass:A1}
The (nonrandom) design matrices $\mathbf X_n,\mathbf H_n\in\mathbb R^{n\times p}$ satisfy
$$
\mathbf Q_n:=\frac1n\mathbf X_n^\top\mathbf X_n \longrightarrow \mathbf Q \succ 0,
\qquad
\mathbf C_n:=\frac1n\mathbf H_n^\top\mathbf H_n \longrightarrow \mathbf C \succeq 0,
$$
as $n\to\infty$. Moreover, no single observation dominates the regression design,
$$
\max_{1\le i\le n}\frac{\|\bm x_{i,n}\|_2^2}{n}\longrightarrow 0,
$$
where $\bm x_{i,n}^\top$ denotes the $i$th row of $\mathbf X_n$. For a fixed $\delta>0$, define the stabilised Gram matrix
$$
\mathbf C_{\delta,n}:=\mathbf C_n+\delta\mathbf I_p \longrightarrow \mathbf C_\delta:=\mathbf C+\delta\mathbf I_p \succ 0.
$$

\item \label{ass:A2}
(\textit{Errors and moments})
The errors $\{\varepsilon_{i,n}\}_{i=1}^n$ are i.i.d.\ with
$$
\mathbb E[\varepsilon_{i,n}]=0,
\qquad
\mathrm{Var}(\varepsilon_{i,n})=\sigma^2\in(0,\infty),
\qquad
\mathbb E\bigl[|\varepsilon_{i,n}|^{2+\eta}\bigr]<\infty
\ \text{for some }\eta>0.
$$

\item \label{ass:A3}
(\textit{Covridge tuning limits and invertibility})
For Covridge, the tuning sequences satisfy $\lambda_{1n}\to\lambda_1\ge 0$ and $\lambda_{2n}\to\lambda_2\ge 0$, and the limit matrix
$$
\mathbf H^\ast_{\mathrm{cov}}:=\mathbf Q+\lambda_1\mathbf C_\delta+\lambda_2\mathbf I_p
$$
is positive definite. Consequently, for all sufficiently large $n$,
$$
\mathbf H^\ast_{n,\mathrm{cov}}:=\mathbf Q_n+\lambda_{1n}\mathbf C_{\delta,n}+\lambda_{2n}\mathbf I_p
$$
is nonsingular and $(\mathbf H^\ast_{n,\mathrm{cov}})^{-1}\to(\mathbf H^\ast_{\mathrm{cov}})^{-1}$.

\item \label{ass:A4}
(\textit{Sparridge local $\ell_1$ scaling})
For Sparridge, let $\gamma_n\ge 0$ denote the $\ell_1$ tuning parameter and assume
$$
\sqrt n\,\gamma_n \longrightarrow \gamma\in[0,\infty).
$$
This local scaling yields a first-order contribution of the $\ell_1$ term when $\gamma>0$, and an asymptotically negligible $\ell_1$ term when $\gamma=0$; see \citet{knight2000asymptotics}.
\end{assump}

Assumption~\ref{ass:A1} imposes deterministic limits for the Gram matrices and excludes a dominant observation, which together with Assumption~\ref{ass:A2} ensures a multivariate central limit theorem for $n^{-1/2}\mathbf X_n^\top\bm\varepsilon_n$. Assumption~\ref{ass:A3} guarantees that the Covridge normal equations remain well-conditioned and that the inverse matrices converge. Assumption~\ref{ass:A4} specifies the local regime under which the $\ell_1$ term affects the first-order asymptotics of Sparridge \citep{knight2000asymptotics}. Under Assumptions~\ref{ass:A1}--\ref{ass:A3}, the Covridge estimator admits a closed-form representation, and its asymptotic distribution follows from an explicit decomposition, a multivariate central limit theorem, and Slutsky's theorem.

\begin{theorem}
\label{thm:covridge-asymptotic}
Consider the fixed-design linear model \( \bm y_n=\mathbf X_n\bm w_0+\bm\varepsilon_n \).
Define the Covridge estimator
$$
\widehat{\bm w}_n
=\arg\min_{\bm w\in\mathbb{R}^p}
\left\{
\frac{1}{2n}\|\bm y_n-\mathbf X_n\bm w\|_2^2
+\frac{\lambda_{1n}}{2}\big\|(\mathbf C_{\delta,n})^{1/2}\bm w\big\|_2^2
+\frac{\lambda_{2n}}{2}\|\bm w\|_2^2
\right\},
$$
and set \(\mathbf Q_n=n^{-1}\mathbf X_n^\top\mathbf X_n\), \(\bm q_n=n^{-1}\mathbf X_n^\top\bm y_n\), and
$$
\mathbf H^\ast_{n,\mathrm{cov}}:=\mathbf Q_n+\lambda_{1n}\mathbf C_{\delta,n}+\lambda_{2n}\mathbf I_p,
\qquad
\bm w_n^\circ:=(\mathbf H^\ast_{n,\mathrm{cov}})^{-1}\mathbf Q_n\bm w_0.
$$
Under Assumptions~\ref{ass:A1}--\ref{ass:A3}, \(\widehat{\bm w}_n=(\mathbf H^\ast_{n,\mathrm{cov}})^{-1}\bm q_n\) and
$$
\sqrt{n}\bigl(\widehat{\bm w}_n-\bm w_n^\circ\bigr)
\xrightarrow{d}
\mathcal{N}\!\Bigl(\bm 0,\;
\sigma^2(\mathbf H^\ast_{\mathrm{cov}})^{-1}\mathbf Q(\mathbf H^\ast_{\mathrm{cov}})^{-1}\Bigr),
$$
where \(\mathbf H^\ast_{\mathrm{cov}}=\mathbf Q+\lambda_1\mathbf C_\delta+\lambda_2\,\mathbf I_p\) is the limit in Assumption~\ref{ass:A3}.
\end{theorem}

\begin{proof}
Since the designs are fixed, $\mathbf Q_n$, $\mathbf C_{\delta,n}$ and $\mathbf H^\ast_{n,\mathrm{cov}}$ are nonrandom.
From $\bm y_n=\mathbf X_n\bm w_0+\bm\varepsilon_n$ we obtain
$$
\bm q_n=\frac{1}{n}\mathbf X_n^\top\bm y_n
=\mathbf Q_n\bm w_0+\frac{1}{n}\mathbf X_n^\top\bm\varepsilon_n.
$$
Hence
$$
\widehat{\bm w}_n
=\bm w_n^\circ+(\mathbf H^\ast_{n,\mathrm{cov}})^{-1}\frac{1}{n}\mathbf X_n^\top\bm\varepsilon_n,
$$
and therefore
$$
\sqrt{n}\bigl(\widehat{\bm w}_n-\bm w_n^\circ\bigr)
=(\mathbf H^\ast_{n,\mathrm{cov}})^{-1}\frac{1}{\sqrt{n}}\mathbf X_n^\top\bm\varepsilon_n.
$$
By Assumptions~\ref{ass:A1}--\ref{ass:A2}, the no-dominant-row condition together with the $(2+\eta)$ moment assumption implies a multivariate Lyapunov (hence Lindeberg--Feller) condition for the triangular array
$\{\bm x_{i,n}\varepsilon_{i,n}:1\le i\le n\}$, and thus
$$
\frac{1}{\sqrt{n}}\mathbf X_n^\top\bm\varepsilon_n
\xrightarrow{d}\mathcal{N}(\bm 0,\sigma^2\mathbf Q).
$$
Moreover, Assumptions~\ref{ass:A1} and \ref{ass:A3} show
$\mathbf H^\ast_{n,\mathrm{cov}}\to\mathbf H^\ast_{\mathrm{cov}}\succ 0$ and hence
$(\mathbf H^\ast_{n,\mathrm{cov}})^{-1}\to(\mathbf H^\ast_{\mathrm{cov}})^{-1}$.
An application of Slutsky's theorem yields the stated limit.
\end{proof}

\begin{remark}
Under Assumptions~\ref{ass:A1} and \ref{ass:A3}, we have $\mathbf Q_n\to\mathbf Q$ and $\mathbf H^\ast_{n,\mathrm{cov}}\to\mathbf H^\ast_{\mathrm{cov}}\succ0$. By continuity of the map $(\mathbf Q,\mathbf H^\ast)\mapsto (\mathbf H^\ast)^{-1}\mathbf Q\bm w_0$, it follows that
$$
\bm w_n^\circ\to \bm w_\star
:= (\mathbf H^\ast_{\mathrm{cov}})^{-1}\mathbf Q\bm w_0.
$$
If, in addition, $\sqrt{n}\,(\bm w_n^\circ-\bm w_\star)\to \bm 0$, then
$$
\sqrt{n}\bigl(\widehat{\bm w}_n-\bm w_\star\bigr)
\xrightarrow{d}
\mathcal{N}\!\Bigl(\bm 0,\;
\sigma^2(\mathbf H^\ast_{\mathrm{cov}})^{-1}\mathbf Q(\mathbf H^\ast_{\mathrm{cov}})^{-1}\Bigr).
$$
Moreover, $\bm w_\star$ is the \emph{penalized (shrunken) target} induced by the Covridge criterion, and in general $\bm w_\star\neq \bm w_0$ unless the penalty vanishes asymptotically (e.g., $\lambda_{1n}\to 0$ and $\lambda_{2n}\to 0$) or an explicit bias correction is applied. Accordingly, the Gaussian limit describes the fluctuations of $\widehat{\bm w}_n$ around $\bm w_\star$, rather than around the true parameter $\bm w_0$.
\end{remark}

\begin{theorem}
\label{thm:sparridge}
Consider the fixed-design linear model \( \bm y_n=\mathbf X_n\bm w_0+\bm\varepsilon_n \). Under Assumptions~\ref{ass:A1}--\ref{ass:A4}, define the Sparridge estimator
$$
\widehat{\bm{w}}_n 
= \arg\min_{\bm{w}\in\mathbb{R}^p}
\left\{
\frac{1}{2n}\|\bm{y}_n - \mathbf{X}_n\bm{w}\|_2^2
+ \frac{\lambda_{1n}}{2}\,\bm{w}^\top \mathbf{C}_{\delta,n}\bm{w}
+ \gamma_n\|\bm{w}\|_1
\right\}.
$$
Let \(\mathbf Q_n:=n^{-1}\mathbf X_n^\top\mathbf X_n\), \(\mathbf H_n:=\mathbf Q_n+\lambda_{1n}\mathbf C_{\delta,n}\), and \(\mathbf H:=\mathbf Q+\lambda_1\mathbf C_\delta\), and define
$$
\bm w_n^\circ := \mathbf H_n^{-1}\mathbf Q_n\bm w_0,
\qquad
\bm w_\star := \mathbf H^{-1}\mathbf Q\bm w_0,
\qquad
A_\star:=\{j:w_{\star j}\neq 0\}.
$$
Let \(\bm Z\sim\mathcal N(\bm 0,\sigma^2\mathbf Q)\), and define, for \(\bm u\in\mathbb R^p\),
$$
\Delta(\bm u)=\frac12\,\bm u^\top \mathbf H \bm u-\bm u^\top \bm Z+\gamma\Bigl(\sum_{j\in A_\star}\mathrm{sign}(w_{\star j})u_j+\sum_{j\notin A_\star}|u_j|\Bigr),
$$
where \(\sqrt n\,\gamma_n\to\gamma\in[0,\infty)\). Then
$$
\sqrt n\,(\widehat{\bm w}_n-\bm w_n^\circ)\xRightarrow{d}\bm U^\star:=\arg\min_{\bm u\in\mathbb R^p}\Delta(\bm u).
$$
In particular, if \(\gamma=0\), then \(\bm U^\star=\mathbf H^{-1}\bm Z\) and \(\bm U^\star\sim\mathcal N\!\bigl(\bm 0,\sigma^2\mathbf H^{-1}\mathbf Q\mathbf H^{-1}\bigr)\).
\end{theorem}
\begin{proof}
We analyze the local behavior of the penalized objective around the finite-sample target \(\bm{w}_n^\circ\). Consider the criterion (up to an additive constant independent of \(\bm{w}\))
$$
L_n(\bm{w})
= \frac{1}{2n}\|\bm{y}_n-\mathbf{X}_n\bm{w}\|_2^2
+ \frac{\lambda_{1n}}{2}\,\bm{w}^\top\mathbf{C}_{\delta,n}\bm{w}
+ \gamma_n\|\bm{w}\|_1,
$$
where \(\bm{w}_n^\circ\) satisfies \(\mathbf{H}_n\bm{w}_n^\circ = \mathbf{Q}_n\bm{w}_0\). Introducing the local reparameterization \(\bm{w}=\bm{w}_n^\circ+\bm{u}/\sqrt{n}\) with \(\bm{u}\in\mathbb{R}^p\), define the rescaled contrast \(\Delta_n(\bm{u}) = n\{L_n(\bm{w}_n^\circ + \bm{u}/\sqrt{n}) - L_n(\bm{w}_n^\circ)\}\).

Using \(\bm{y}_n=\mathbf{X}_n\bm{w}_0+\bm{\varepsilon}_n\), a direct expansion of the quadratic loss yields
$$
\begin{aligned}
\Delta_n(\bm u)
&= -\frac{1}{\sqrt n}\bm u^\top \mathbf X_n^\top\!\bigl(\bm y_n-\mathbf X_n\bm w_n^\circ\bigr)
   +\frac12\,\bm u^\top \mathbf Q_n \bm u
   + R_n(\bm u) \\
&\quad + n\frac{\lambda_{1n}}{2}\Bigl\{
\bigl(\bm w_n^\circ+\tfrac{\bm u}{\sqrt n}\bigr)^\top \mathbf C_{\delta,n}
\bigl(\bm w_n^\circ+\tfrac{\bm u}{\sqrt n}\bigr)
-(\bm w_n^\circ)^\top \mathbf C_{\delta,n}\bm w_n^\circ
\Bigr\},
\end{aligned}
$$
where \(R_n(\bm{u}) := n\gamma_n\{\|\bm{w}_n^\circ + \bm{u}/\sqrt{n}\|_1 - \|\bm{w}_n^\circ\|_1\}\) denotes the \(\ell_1\) contribution. Define also
$$
R(\bm u):=\gamma\Bigl(\sum_{j\in A_\star}\mathrm{sign}(w_{\star j})u_j+\sum_{j\notin A_\star}|u_j|\Bigr).
$$

Moreover,
$$
\mathbf{X}_n^\top(\bm{y}_n-\mathbf{X}_n\bm{w}_n^\circ)
= \mathbf{X}_n^\top\bm{\varepsilon}_n + n\mathbf{Q}_n(\bm{w}_0-\bm{w}_n^\circ).
$$
Since \(\mathbf{Q}_n\bm{w}_0=\mathbf{H}_n\bm{w}_n^\circ\), it follows that \(\mathbf{Q}_n(\bm{w}_0-\bm{w}_n^\circ)=\lambda_{1n}\mathbf{C}_{\delta,n}\bm{w}_n^\circ\), and hence
$$
\mathbf{X}_n^\top(\bm{y}_n-\mathbf{X}_n\bm{w}_n^\circ)
= \mathbf{X}_n^\top\bm{\varepsilon}_n + n\lambda_{1n}\mathbf{C}_{\delta,n}\bm{w}_n^\circ.
$$
A corresponding expansion of the quadratic penalty shows that the linear terms in

$\sqrt{n}\lambda_{1n}\bm{u}^\top\mathbf{C}_{\delta,n}\bm{w}_n^\circ$ cancel exactly between the loss and the penalty, so that the smooth part of \(\Delta_n(\bm{u})\) reduces to
$$
-\bm{u}^\top\bm{Z}_n + \frac{1}{2}\bm{u}^\top\mathbf{H}_n\bm{u},
\qquad
\bm{Z}_n := \frac{1}{\sqrt{n}}\mathbf{X}_n^\top\bm{\varepsilon}_n,
\qquad
\mathbf{H}_n=\mathbf{Q}_n+\lambda_{1n}\mathbf{C}_{\delta,n}.
$$

Under Assumptions~\ref{ass:A1}--\ref{ass:A2}, \(\bm Z_n\xRightarrow{d}\bm Z\sim\mathcal N(\bm 0,\sigma^2\mathbf Q)\), and Assumption~\ref{ass:A1} together with $\lambda_{1n}\to\lambda_1$ implies $\mathbf H_n\to\mathbf H$ and hence $\bm w_n^\circ\to \bm w_\star$. The one-dimensional argument of \citet{knight2000asymptotics} then yields, for each fixed \(\bm u\),
$$
R_n(\bm u)\xrightarrow{p} R(\bm u),
$$
with convergence uniform on compact sets. Consequently, \(\Delta_n(\cdot)\to\Delta(\cdot)\) uniformly on compacts in probability (equivalently, epi-convergence in probability for convex functions), and since \(\mathbf H\succ0\) the limit criterion is strictly convex with unique minimizer \(\bm U^\star\). The convexity lemma and the argmin continuous mapping theorem \citep{geyer1994asymptotics,knight2000asymptotics} therefore imply
$$
\widehat{\bm{u}}_n:=\arg\min_{\bm u\in\mathbb R^p}\Delta_n(\bm u)
=\sqrt{n}(\widehat{\bm w}_n-\bm w_n^\circ)
$$
If \(\gamma=0\), the nonsmooth term vanishes and the minimizer is \(\bm U^\star=\mathbf H^{-1}\bm Z\), yielding the stated Gaussian law.
\end{proof}

\begin{remark}
The local scaling $\sqrt{n}\gamma_n\to\gamma$ determines whether the $\ell_1$ penalty contributes at the first-order asymptotic scale.
When $\gamma>0$, the non-smooth term remains present in the limiting criterion, leading to the generally non-Gaussian limit
\(\bm{U}^\star=\arg\min_{\bm u\in\mathbb R^p}\Delta(\bm u)\) for \(\sqrt{n}(\widehat{\bm{w}}_n-\bm{w}_n^\circ)\).
If $\gamma=0$, the $\ell_1$ contribution vanishes, $\bm{U}^\star=\mathbf{H}^{-1}\bm{Z}$, and the limit is Gaussian with covariance
\(\sigma^2\mathbf{H}^{-1}\mathbf{Q}\mathbf{H}^{-1}\). In this case, the Sparridge criterion reduces to a purely quadratic covariance-regularized least-squares problem, and the corresponding centered asymptotic distribution coincides with that of the covariance-based ridge criterion with $\lambda_{2n}=0$.
\end{remark}


\section{Simulation Study} \label{sec:simulation}

We conduct a simulation study to evaluate the performance of several regularized methods under different conditions, such as feature correlation, noise, and sparsity. The primary aim of the simulation study is to evaluate the performance of each method in producing reliable predictions in both low- and high-dimensional settings.

\subsection{Data Generating Process}
Let $n$ denote the number of observations, $p$ the total number of features, and $k$ the number of informative features. We generate data according to $y_i = f^\star(\mathbf{x}_i) + \varepsilon_i, \, 
\varepsilon_i \sim \mathcal{N}(0, \sigma^2),$
where $\mathbf{x}_i \in \mathbb{R}^{p}$ denotes the predictor vector for observation $i$, and $f^\star$ is the true regression function. The error terms are assumed to be independent and identically distributed and independent of the predictors. Among the $p$ predictors, only the first $k$ are informative. To introduce dependence among the informative predictors, we assume that they follow a multivariate normal distribution with covariance matrix \(\boldsymbol{\Sigma}\) whose elements $\sigma_{j\ell} =1,$ if $j=\ell$ and $\sigma_{j\ell}=\rho$ if $j\neq \ell$ for $j,\ell = 1,\dots,k,$ following \citep{yang2024new}, where  $\rho$ controls the correlation strength among the informative features. An equivalent way to generate this covariance structure is through the latent-factor representation \citep{qasim2022new}. The remaining predictors are generated independently and treated as noninformative noise variables. In the linear setting, the true regression function is defined as
$
f^\star(\mathbf{x}_i) = \mathbf{x}_i^\top \boldsymbol{\theta}^\star,
$
where $\boldsymbol{\theta}^\star \in \mathbb{R}^p$ is the true coefficient vector with
$
\theta_j^\star \stackrel{\text{i.i.d.}}{\sim} \mathcal{N}(0,\tau^2),
$ for $j = 1,\dots,k,$ while the remaining coefficients are set to zero, i.e., $\theta_j^* = 0$ for $j = k+1,\dots,p$.
To assess the robustness under model misspecification, we also consider a nonlinear variant,
$
f^\star(\mathbf{x}_i) = \sum^k_{j=1} \theta_j^\star \sin(x_{ij}),
$
which introduces a smooth nonlinear structure with the same sparsity pattern (see, e.g., \cite{park2022sparse}).
We examine several data-generating processes (DGPs) with different choices of $(n,p,k)$, feature correlation $\rho \in \{0.25, 0.75\}$, and noise level $\sigma \in \{0.10, 2.00\}$. Three representative scenarios are given below: 
\begin{enumerate}[label=(\roman*)]
\item DGP1: $(200, 20, 10)$  corresponding to a low-dimensional setting ($p<n$) with a partially sparse signal;
\item DGP2: $(1000, 200, 100)$ representing a larger scale low-dimensional setting ($p<n$) with a moderately dense signal;
\item DGP3: $(500, 2000, 100)$ describes a high-dimensional setting ($p \gg n$) with a relatively sparse signal.
\end{enumerate}
Each scenario is evaluated with both linear and nonlinear specifications of the regression function $f^\star$. To see the robustness of the proposed methods, other DGP scenarios are also considered, which includes very low sparsity level and larger sample sizes compared to $k$, to address high dimensionality. Results for these cases are reported in the Appendix.

\subsection{Model specifications}
We compare six penalization strategies. In all cases, the underlying model is a neural network, and the differences lie in the form of regularization applied to the network parameters. Specifically, we consider:   
(1) no regularization,  
(2) weight decay (ridge regression),  
(3) Lasso,  
(4) Elastic Net,  
(5) weight decay by integrating covariance structure among features (Covridge), and  
(6) two-parameter regularization combining sparsity and covariance penalties (Sparridge).  
Neural networks are implemented using a fixed architecture consisting of two fully connected hidden layers with 64 and 32 units, respectively. ReLU activation functions are used in both hidden layers. The output layer contains a single neuron for the regression output. We trained the model using the Adam optimizer, taking advantage of its adaptive learning rate and momentum mechanisms, and used the default hyperparameter settings throughout. Networks were trained for fixed\footnote{The simulation design could be extended to explore alternative settings, such as different numbers of epochs, batch sizes, or hidden-layer neurons. However, for this study, we fixed these hyperparameters to maintain a controlled experimental setup and to evaluate the performance of the methods consistently based on different DGPs. Exploring alternative hyperparameter choices, such as selecting the number of neurons via cross-validation to minimize MSE, is possible but was not used here due to computational resources and time constraints. The authors acknowledge that different settings may yield different absolute performance, but the relative comparison of methods remains informative under the fixed design.} epochs using a batch size of 32 and the mean squared error (MSE) loss function.
For each replication, the dataset was split into a training (75\%) and test (25\%) set, with features standardized based on the training data.  
Model performance was evaluated using MSE, mean absolute error (MAE), and prediction bias, averaged over 100 replications.
Hyperparameters were selected via k-fold cross-validation on the training data. For all regularization methods, tuning was carried out over the common grid of candidate values $\{0.001, 0.01, 0.1, 0.5, 0.9\}$. 
For single-parameter methods (Ridge, Lasso), the corresponding tuning parameter $\lambda$ was evaluated over this grid. For two-parameter methods (Elastic Net, Covridge, and Sparridge), all possible combinations of tuning parameters were evaluated using the same grid. The optimal hyperparameters were chosen by minimizing the cross-validated MSE, and the final networks were retrained on the full training sample using these selected values.

\subsection{Computer specifications}
All the computations in this study were performed on the Alvis high-performance computing cluster at C3SE, Sweden (\url{https://alvis.c3se.chalmers.se/}), a national NAISS resource. Simulations were performed on nodes equipped with dual Intel Xeon Gold 6338 CPUs running at 2.00 GHz, providing a total of 64 CPU cores per node and 256 GiB of RAM. GPU-accelerated computations were conducted on nodes equipped with four NVIDIA GPUs (T4, V100, A40, or A100, depending on availability), with 48 GiB VRAM per A40 GPU, using CUDA version 13.0 and NVIDIA driver version. All analyses were performed in Python, using \texttt{TensorFlow/Keras} for neural networks modeling and \texttt{scikit-learn} for data preprocessing and evaluation.

\subsection{Simulation results}

This subsection summarizes the results of the simulation study and compares the performance of the competing methods across the considered data-generating scenarios. Tables \ref{tab:simdgp1}--\ref{tab:simdgp3} present the empirical performance of various regression methods for DGP1--DGP3. It can be seen that, among all DGPs, regularization consistently improves predictive performance over the unregularized baseline, indicating the effect of shrinkage in high-dimensional settings and under both moderate and high correlation among predictors, as well as in the presence of noise. Ridge, Lasso, and Elastic Net generally improve upon the unregularized method, although their relative performance depends on the correlation structure and sparsity of the true signal. Ridge regression tends to perform better in highly correlated predictor settings $(\rho=0.75)$ or in dense signal designs, whereas Lasso performs better when the underlying coefficient vector is sparse. The Covridge and Sparridge estimators outperform Ridge and Lasso under certain conditions. Among nearly all DGPs, these methods provide the lowest or near-lowest MSE and MAE values. The improvement is significant in nonlinear and high-dimensional cases. Sparridge often yields the smallest errors, suggesting that its adaptive regularization promotes a favorable bias–variance tradeoff, especially when the true model exhibits partial sparsity or correlated signal groups.

The impact of omitting regularization is evident in the unregularized results, which often exhibit higher test MSE and MAE due to overfitting, suggesting that the network is fitting noise rather than capturing the underlying signal. As the noise level increases from 0.10 to 2.00, models with regularization show substantially lower performance loss in MSE, demonstrating their robustness to noise and multicollinearity.  Under DGP1 (Table \ref{tab:simdgp1}), both linear and nonlinear low-dimensional settings show that Covridge and Sparridge achieve the smallest test MSEs, even when predictors are highly correlated. Covridge and Sparridge also show consistent improvements over Elastic Net under both low- and high-noise settings, highlighting the benefits of covariance-based shrinkage. On the other hand, under DGP2 (Table \ref{tab:simdgp2}), with moderate sparsity and complex nonlinearities, Elastic Net and Lasso perform well, but Sparridge achieves the lowest or near-lowest prediction errors under both noise conditions.

DGP3 (Table \ref{tab:simdgp3}) corresponds to high-dimensional settings $p \gg n$. In this scenario, Sparridge remains stable and consistently achieves the lowest MSEs, particularly. By contrast, Lasso and Ridge perform well mainly in linear settings but exhibit higher MSE compared to Sparridge. Sparridge demonstrates strong performance across both linear and nonlinear models, benefiting from adaptive, data-informed regularization, and achieves lower prediction MSE than Lasso in nonlinear settings. This robustness suggests that Sparridge regularization robustly controls sparsity and bias. Covridge, however, performs poorly in high-dimensional linear models with many irrelevant predictors (Table \ref{tab:simdgp3}). This limitation appears to stem from its reliance on covariance-based weighting without an $\ell_1$ component, which limits its ability to exploit sparsity effectively. In our simulations, the unregularized method exhibits the highest test MSE, particularly in high-dimensional or in the presence of strong predictor correlation. Ridge and Lasso both reduce prediction error, with Lasso promoting sparsity, while Elastic Net balances sparsity and stability but does not account for feature covariance. Covridge incorporates covariance information into the penalty and therefore improves stability in correlated or dense predictor regimes. Sparridge combines sparsity-inducing regularization with covariance-aware shrinkage, leading to improved prediction performance and robustness. These results show that covariance-informed penalties offer clear advantages over standard $\ell_1$ and $\ell_2$ regularization in settings with strong predictor dependencies or high model complexity. Although the values of the performance measures (MSE, MAE, and bias) vary with the choice of hyperparameters, the relative performance ranking of the methods remains broadly consistent across the considered scenarios

\begin{table}[htbp]
\centering
\caption{Performance metrics (MSE, MAE, Bias) for different methods under two noise levels (0.10 and 2.00), correlation levels ($\rho=0.25$ and $0.75$), and linear/nonlinear models for DGP1.}
\resizebox{\textwidth}{!}{%
\begin{tabular}{lcccccccccccc}
\toprule
\multirow{2}{*}{Method} & \multicolumn{6}{c}{$\rho=0.25$} & \multicolumn{6}{c}{$\rho=0.75$} \\
\cmidrule(lr){2-7} \cmidrule(lr){8-13}
 & \multicolumn{2}{c}{MSE} & \multicolumn{2}{c}{MAE} & \multicolumn{2}{c}{Bias} & \multicolumn{2}{c}{MSE} & \multicolumn{2}{c}{MAE} & \multicolumn{2}{c}{Bias} \\
\cmidrule(lr){2-3} \cmidrule(lr){4-5} \cmidrule(lr){6-7}
\cmidrule(lr){8-9} \cmidrule(lr){10-11} \cmidrule(lr){12-13}
 & 0.10 & 2.00 & 0.10 & 2.00 & 0.10 & 2.00 & 0.10 & 2.00 & 0.10 & 2.00 & 0.10 & 2.00 \\
\midrule
\multicolumn{13}{l}{{Linear}} \\
\midrule
No Regularization & 0.817 & 5.165 & 0.703 & 1.731 & -0.173 & -0.366 & 1.018 & 5.405 & 0.780 & 1.745 & -0.397 & -0.547 \\
Ridge  & 0.766 & 5.004 & 0.668 & 1.686 & -0.172 & -0.412 & 0.939 & 5.238 & 0.757 & 1.717 & -0.387 & -0.566 \\
Lasso  & 0.817 & 4.810 & 0.672 & 1.652 & -0.185 & -0.387 & 1.006 & 5.219 & 0.783 & 1.724 & -0.398 & -0.581 \\
Elastic Net       & 0.717 & 4.774 & 0.641 & 1.646 & -0.181 & -0.384 & 0.925 & 5.194 & 0.758 & 1.707 & -0.390 & -0.577 \\
Covridge          & 0.683 & 4.399 & 0.560 & 1.571 & -0.141 & -0.375 & 0.728 & 4.629 & 0.600 & 1.619 & -0.274 & -0.489 \\
Sparridge          & 0.666 & 4.246 & 0.567 & 1.558 & -0.132 & -0.369 & 0.668 & 4.578 & 0.577 & 1.620 & -0.245 & -0.466 \\
\midrule
\multicolumn{13}{l}{{Nonlinear}} \\
\midrule
No Regularization & 0.367 & 5.123 & 0.481 & 1.772 & -0.056 & -0.215 & 0.333 & 4.940 & 0.460 & 1.740 & -0.084 & -0.179 \\
Ridge  & 0.233 & 4.795 & 0.372 & 1.717 & -0.049 & -0.175 & 0.197 & 4.888 & 0.351 & 1.726 & -0.074 & -0.213 \\
Lasso & 0.311 & 4.739 & 0.441 & 1.712 & -0.070 & -0.206 & 0.262 & 4.921 & 0.407 & 1.733 & -0.083 & -0.194 \\
Elastic Net       & 0.224 & 4.898 & 0.365 & 1.732 & -0.047 & -0.165 & 0.207 & 4.864 & 0.361 & 1.718 & -0.071 & -0.217 \\
Covridge          & 0.180 & 4.512 & 0.306 & 1.652 & -0.009 & -0.160 & 0.126 & 4.342 & 0.271 & 1.617 & -0.041 & -0.168 \\
Sparridge          & 0.188 & 4.343 & 0.319 & 1.633 & -0.008 & -0.147 & 0.133 & 4.286 & 0.279 & 1.611 & -0.046 & -0.173 \\
\bottomrule
\end{tabular}%
}
\label{tab:simdgp1}
\end{table}

\begin{table}[H]
\centering
\caption{Performance metrics (MSE, MAE, Bias) for different methods under two noise levels (0.10 and 2.00), correlation levels ($\rho=0.25$ and $0.75$), and linear/nonlinear models for DGP2.}
\resizebox{\textwidth}{!}{%
\begin{tabular}{lcccccccccccc}
\toprule
\multirow{2}{*}{Method} & \multicolumn{6}{c}{$\rho=0.25$} & \multicolumn{6}{c}{$\rho=0.75$} \\
\cmidrule(lr){2-7} \cmidrule(lr){8-13}
 & \multicolumn{2}{c}{MSE} & \multicolumn{2}{c}{MAE} & \multicolumn{2}{c}{Bias} & \multicolumn{2}{c}{MSE} & \multicolumn{2}{c}{MAE} & \multicolumn{2}{c}{Bias} \\
\cmidrule(lr){2-3} \cmidrule(lr){4-5} \cmidrule(lr){6-7}
\cmidrule(lr){8-9} \cmidrule(lr){10-11} \cmidrule(lr){12-13}
 & 0.10 & 2.00 & 0.10 & 2.00 & 0.10 & 2.00 & 0.10 & 2.00 & 0.10 & 2.00 & 0.10 & 2.00 \\
\midrule
\multicolumn{13}{l}{\textbf{Linear}} \\
\midrule
No Regularization & 35.733 & 41.460 & 4.691 & 5.085 & 0.105 & 0.178 & 23.805 & 28.210 & 3.820 & 4.180 & 0.038 & 0.138 \\
Ridge   & 8.034 & 17.811 & 1.966 & 3.143 & 0.004 & 0.102 & 14.498 & 18.867 & 2.890 & 3.319 & -0.050 & 0.033 \\
Lasso  & 4.488 & 12.197 & 1.212 & 2.542 & 0.020 & 0.001 & 10.923 & 16.286 & 2.568 & 3.128 & 0.139 & 0.145 \\
Elastic Net       & 4.190 & 12.230 & 1.222 & 2.546 & 0.002 & 0.001 & 10.797 & 16.247 & 2.551 & 3.119 & 0.111 & 0.104 \\
Covridge          & 5.654 & 13.217 & 1.864 & 2.826 & -0.069 & 0.051 & 6.407 & 16.590 & 1.981 & 3.221 & 0.013 & 0.054 \\
Sparridge          & 5.716 & 11.572 & 1.793 & 2.679 & -0.015 & -0.021 & 6.324 & 12.436 & 1.971 & 2.797 & -0.057 & 0.065 \\
\midrule
\multicolumn{13}{l}{\textbf{Nonlinear}} \\
\midrule
No Regularization & 12.135 & 17.556 & 2.740 & 3.341 & -0.080 & -0.025 & 8.563 & 14.263 & 2.310 & 2.993 & -0.155 & -0.133 \\
Ridge   & 1.558 & 8.462 & 0.831 & 2.294 & -0.038 & 0.020 & 1.191 & 7.974 & 0.727 & 2.216 & -0.091 & -0.064 \\
Lasso  & 1.679 & 8.437 & 0.896 & 2.259 & -0.037 & 0.013 & 1.028 & 7.611 & 0.547 & 2.131 & -0.078 & -0.086 \\
Elastic Net       & 1.287 & 8.398 & 0.694 & 2.256 & -0.025 & -0.031 & 1.000 & 8.178 & 0.596 & 2.149 & -0.067 & -0.178 \\
Covridge          & 2.870 & 9.176 & 1.323 & 2.388 & 0.031 & -0.023 & 1.876 & 8.612 & 1.018 & 2.305 & -0.076 & -0.059 \\
Sparridge          & 2.917 & 8.151 & 1.241 & 2.265 & -0.067 & 0.013 & 1.386 & 7.220 & 0.828 & 2.130 & -0.088 & -0.086 \\
\bottomrule
\end{tabular}%
}
\label{tab:simdgp2}
\end{table}


\begin{table}[H]
\centering
\caption{Performance metrics (MSE, MAE, Bias) for different methods under two noise levels (0.10 and 2.00), correlation levels ($\rho=0.25$ and $0.75$), and linear/nonlinear models for DGP3.}
\resizebox{\textwidth}{!}{%
\begin{tabular}{lcccccccccccc}
\toprule
\multirow{2}{*}{Method} & \multicolumn{6}{c}{$\rho=0.25$} & \multicolumn{6}{c}{$\rho=0.75$} \\
\cmidrule(lr){2-7} \cmidrule(lr){8-13}
 & \multicolumn{2}{c}{MSE} & \multicolumn{2}{c}{MAE} & \multicolumn{2}{c}{Bias} & \multicolumn{2}{c}{MSE} & \multicolumn{2}{c}{MAE} & \multicolumn{2}{c}{Bias} \\
\cmidrule(lr){2-3} \cmidrule(lr){4-5} \cmidrule(lr){6-7}
\cmidrule(lr){8-9} \cmidrule(lr){10-11} \cmidrule(lr){12-13}
 & 0.10 & 2.00 & 0.10 & 2.00 & 0.10 & 2.00 & 0.10 & 2.00 & 0.10 & 2.00 & 0.10 & 2.00 \\
\midrule
\multicolumn{13}{l}{\textbf{Linear}} \\
\midrule
No Regularization & 105.781 & 260.417 & 7.725 & 12.220 & 0.916 & 1.135 & 99.508 & 105.251 & 7.532 & 7.717 & 0.928 & 0.917 \\
Ridge   & 90.684 & 238.620 & 7.239 & 11.789 & 0.680 & 1.204 & 82.541 & 90.408 & 6.969 & 7.237 & 0.733 & 0.734 \\
Lasso   & 54.655 & 148.675 & 5.712 & 9.352 & 0.737 & 1.137 & 48.016 & 55.138 & 5.352 & 5.714 & 0.846 & 0.711 \\
Elastic Net       & 55.059 & 148.596 & 5.730 & 9.360 & 0.719 & 1.159 & 48.523 & 55.070 & 5.403 & 5.726 & 0.821 & 0.702 \\
Covridge          & 100.532 & 247.614 & 7.502 & 12.045 & 0.657 & 0.878 & 91.906 & 101.342 & 7.274 & 7.513 & 0.725 & 0.641 \\
Sparridge          & 52.284 & 140.897 & 5.478 & 9.030 & 0.733 & 1.178 & 46.690 & 52.754 & 5.165 & 5.516 & 0.701 & 0.769 \\
\midrule
\multicolumn{13}{l}{\textbf{Nonlinear}} \\
\midrule
No Regularization & 133.554 & 141.590 & 9.127 & 9.341 & 0.897 & 0.938 & 68.726 & 75.126 & 6.424 & 6.727 & 0.399 & 0.386 \\
Ridge   & 118.516 & 126.212 & 8.651 & 8.872 & 0.717 & 0.729 & 60.782 & 67.915 & 6.075 & 6.483 & 0.195 & 0.204 \\
Lasso  & 64.477  & 70.197  & 6.363 & 6.592 & 0.203 & 0.156 & 38.934 & 43.902 & 4.927 & 5.212 & 0.180 & 0.194 \\
Elastic Net       & 64.196  & 69.841  & 6.352 & 6.566 & 0.167 & 0.124 & 38.647 & 45.412 & 4.904 & 5.274 & 0.234 & 0.205 \\
Covridge          & 123.573 & 133.362 & 8.806 & 9.075 & 0.719 & 0.706 & 63.058 & 71.300 & 6.191 & 6.620 & 0.184 & 0.226 \\
Sparridge          & 54.532  & 61.408  & 5.870 & 6.220 & 0.327 & 0.192 & 29.416 & 37.425 & 4.278 & 4.827 & 0.315 & 0.150 \\
\bottomrule
\end{tabular}%
}
\label{tab:simdgp3}
\end{table}

\section{Real-World Applications} \label{sec:application}
\subsection{Energy Dataset: Cooling load prediction in buildings}
In this section, we evaluate the performance of different regularized neural network regression methods on a real-world energy dataset, focusing on predicting the cooling load of buildings. The dataset is obtained from the UCI Machine Learning Repository \citep{tsanas2012accurate}. The dataset consists of 768 simulated buildings generated using Ecotect software, representing twelve different building shapes. These buildings vary in glazing area, glazing distribution, orientation, and other architectural parameters. The dataset includes eight input features and one target variable (cooling load, measured in kWh/m$^2$). The features are: relative compactness (X1), surface area (X2), wall area (X3), roof area (X4), overall height (X5), orientation (X6), glazing area (X7), and glazing area distribution (X8). Figure~\ref{fig:all_features} presents scatter plots showing the relationships between cooling load and each input variable. The nonlinear nature of the dataset motivates the use of neural networks but makes them prone to overfitting, particularly when the available sample size is limited. Therefore, this application focuses on evaluating how different regularization strategies can control overfitting and enhance generalization. While neural networks are often regarded as black-box models, applying these regularization methods provides better interpretability in terms of model stability and predictive robustness.

\begin{figure}[ht]
\centering
\begin{subfigure}[b]{0.23\textwidth}
    \includegraphics[width=\textwidth]{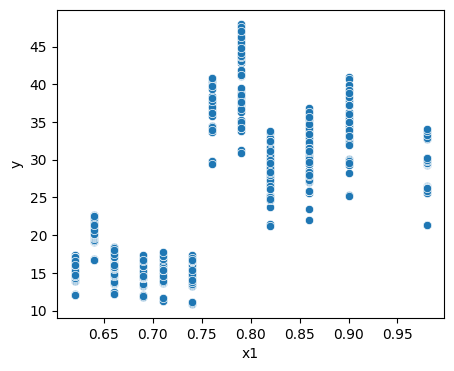}
    \caption{X1: Compactness}
\end{subfigure}
\begin{subfigure}[b]{0.23\textwidth}
    \includegraphics[width=\textwidth]{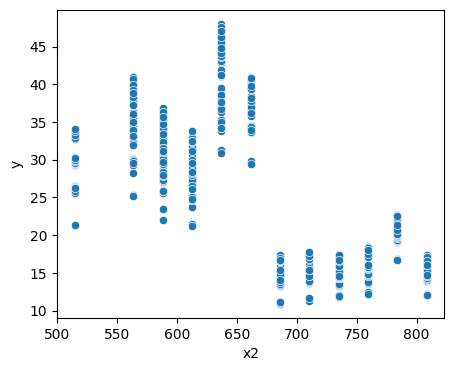}
    \caption{X2: Surface area}
\end{subfigure}
\begin{subfigure}[b]{0.23\textwidth}
    \includegraphics[width=\textwidth]{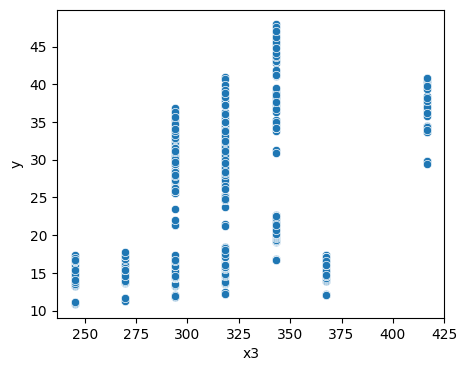}
    \caption{X3: Wall area}
\end{subfigure}
\begin{subfigure}[b]{0.23\textwidth}
    \includegraphics[width=\textwidth]{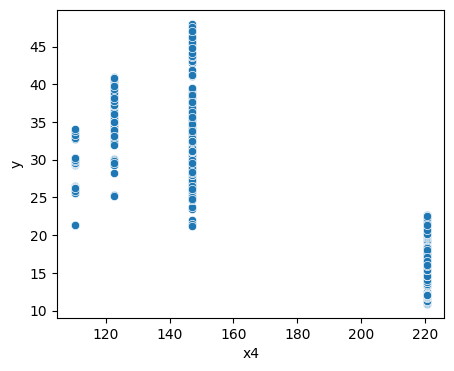}
    \caption{X4: Roof area}
\end{subfigure}

\begin{subfigure}[b]{0.23\textwidth}
    \includegraphics[width=\textwidth]{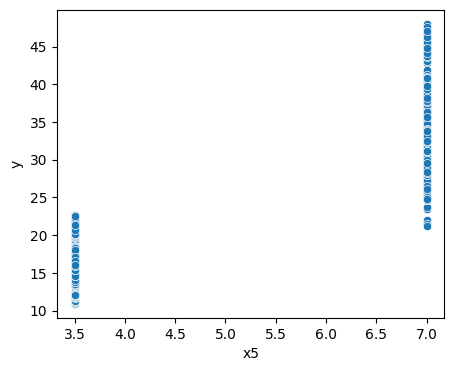}
    \caption{X5: Height}
\end{subfigure}
\begin{subfigure}[b]{0.23\textwidth}
    \includegraphics[width=\textwidth]{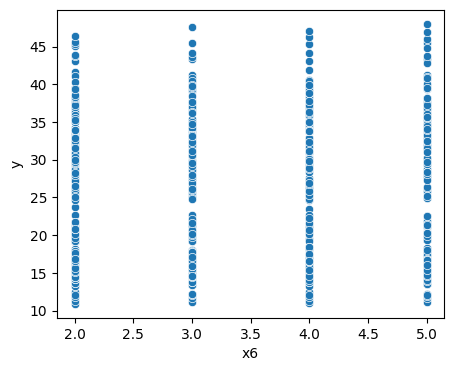}
    \caption{X6: Orientation}
\end{subfigure}
\begin{subfigure}[b]{0.23\textwidth}
    \includegraphics[width=\textwidth]{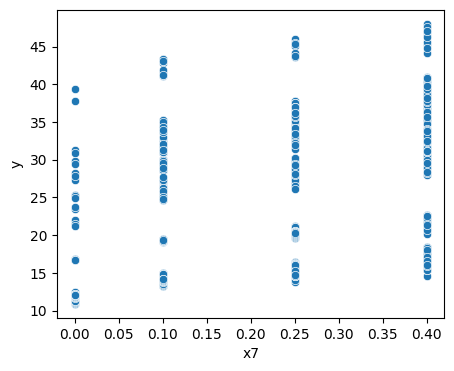}
    \caption{X7: Glazing area}
\end{subfigure}
\begin{subfigure}[b]{0.23\textwidth}
    \includegraphics[width=\textwidth]{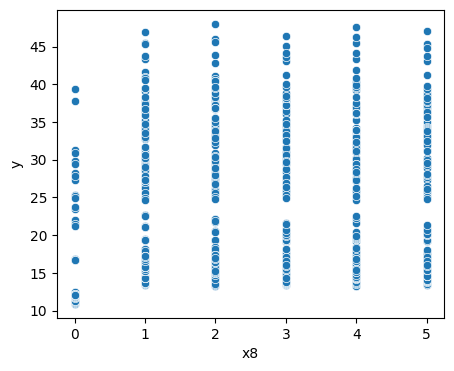}
    \caption{X8:  Pattern}
\end{subfigure}
\caption{Scatter plots of cooling load versus all eight input features. }
\label{fig:all_features}
\end{figure}

The dataset shows strong nonlinearity, confirmed by the Ramsey regression equation specification error test \citep{ramsey1969tests} ($F=48.90$, $p<10^{-12}$). This makes it an appropriate benchmark for evaluating regularized neural networks. We trained a feedforward neural network with two hidden layers, 64 neurons in the first layer and 32 in the second. The number of hidden layers and neurons per layer can be large in modern statistical learning settings \citep{gareth2023introduction}. Complex neural network architectures allow the model to capture highly nonlinear patterns, but they also increase the risk of overfitting. Regularization techniques are therefore essential to control model complexity and ensure stable generalization. The ReLU activation function and Adam optimizer were used to train the neural network. The network was trained for 500 epochs, with 30\% of the data reserved for testing. All neural networks were trained using MSE as the loss function. Regularization hyperparameters were tuned through a grid search using 10-fold cross-validation. Model performance was evaluated using MSE, MAE, RMSE, and $R^2$. The results are summarized in Table~\ref{tab:real_results}.

\begin{table}[ht]
\centering
\caption{Performance comparison of different methods on the cooling load prediction task.}
\begin{tabular}{lcccc}
\hline
Method & MSE & MAE & R$^2$ & RMSE\\
\hline
No Regularization & 1.924 & 0.949 & 0.979 & 1.387 \\
Ridge & 1.716 & 0.862 & 0.981 & 1.310 \\
Lasso & 1.877 & 0.937 & 0.979 & 1.370 \\
Elastic Net & 1.810 & 0.909 & 0.980 & 1.345 \\
Covridge & 1.494 & 0.840 & 0.984 & 1.222 \\
Sparridge & 1.684 & 0.873 & 0.982 & 1.298 \\
\hline
\end{tabular}
\label{tab:real_results}
\end{table}

As shown in Table~\ref{tab:real_results}, regularized (particularly Covridge and Sparridge) methods achieved the lowest MSE and the highest $R^2$, indicating superior generalization compared to the no regularization. These methods captured the nonlinear dependencies present in the dataset, achieving better predictive accuracy than reported in previous studies \citep{tsanas2012accurate, xu2022prediction}.  The performance ranking of the regularized methods can vary depending on the tuning of hyperparameters, network architecture, and training epochs. The number of training epochs was set to 500 for neural networks to ensure comparability across methods. Across all settings, regularization yields substantial gains in predictive accuracy relative to the unregularized baseline. While Lasso attains the lowest test MSE under early stopping in some configurations, Covridge and Sparridge remain highly competitive, particularly as model complexity increases. This pattern is consistent with the simulation evidence, where the proposed geometry-aware penalties provide significant improvements in high-dimensional settings. These findings suggest that practitioners should carefully select model configurations to achieve optimal predictive performance. This application demonstrates that proposed covariance-based regularization methods, such as Covridge and Sparridge, can substantially improve the generalization performance of deep neural networks when modeling complex nonlinear energy datasets.

\subsection{High-dimensional biological classification}

The benefit of regularization in high-dimensional settings is illustrated using the GSE9476 microarray gene expression dataset \citep{feltes2019cumida}, which is used here for the classification of leukemia and progenitor cell types. The dataset is publicly available on Kaggle\footnote{https://www.kaggle.com/datasets/brunogrisci/leukemia-gene-expression-cumida/data} and consists of $64$ samples measured over $22000$ gene expression features. The primary objective is a five-class classification problem distinguishing between bone marrow, peripheral blood, CD34+ cell lines, and acute myeloid leukemia samples. Such data present a challenging learning scenario due to extreme feature dimensionality relative to sample size, making regularization important for controlling model complexity and providing reliable generalization. Before model training, the genes exhibiting the strongest associations with the class labels are selected using the ANOVA F test. All selected 2000 features were standardized to a zero mean and unit variance. A multi-layer perceptron with two hidden layers of 8 and 4 units was used as the base classifier. We apply ReLU activation functions and a softmax activation at the output layer. The network was trained using the Adam optimizer with a batch size of $16$ for up to $500$ epochs, and early stopping based on validation loss (patience = 10) was applied to prevent overfitting and retain the model parameters corresponding to the best performance.

 We applied a two-stage evaluation procedure wherein the first stage, the regularization parameter $\lambda$ for each method was optimized using k-fold cross-validation over a candidate set of values ($\lambda \in [0.0001, 1.0]$). For Lasso ($\ell_1$) and Ridge ($\ell_2$) regularization, a single penalty parameter $\lambda$ was tuned. Whereas, Elastic Net, Covridge, and Sparridge were formulated using two independent regularization parameters. In the second stage, the optimally tuned hyperparameters were evaluated using a robust repeated cross-validation scheme consisting of $10$ repetitions of k-fold cross-validation, minimizing dependence on a single random data split. The primary performance metric used is balanced accuracy, which measures class imbalance by computing the average accuracy among all classes. Given the small sample size and unequal class counts in this five-class problem, this metric provides a robust assessment of the model's ability to correctly classify each cell type. In this way, it reduces bias from class imbalance by averaging the recall obtained for each class. The robust performance results ($\text{mean} \pm \text{standard deviation}$) of the neural network classifiers, measured by test accuracy for each method, are summarized in Table \ref{tab:cv_results}.

\begin{table}[tb]
\centering
\caption{Robust performance of neural network classifiers on the GSE9476 dataset}
\begin{tabular}{lcc}
\toprule
{Method} & {Accuracy ($\mu \pm \sigma$)} \\
\midrule
No Regularization &  $0.8415 \pm 0.0610$ \\
Lasso   & $0.9671 \pm 0.0202$ \\
Ridge   & $0.8664 \pm 0.0438$ \\
Elastic Net & $0.9181 \pm 0.0423$ \\
Covridge  &  $0.9718 \pm 0.0199$ \\
Sparridge  & $0.9673 \pm 0.0255$ \\  
\bottomrule
\end{tabular}
\label{tab:cv_results}
\end{table}
The results in Table~\ref{tab:cv_results} demonstrate the value of regularization in high-dimensional gene expression classification. Regularization not only improves average test accuracy but also reduces variability across repeated cross-validation, as shown by the smaller standard deviations, particularly for Covridge ($\sigma=0.0199$).  Lasso achieves a good test accuracy ($0.9671 \pm 0.0202$) relative to Ridge ($0.8664 \pm 0.0438$), indicating the advantage of sparsity over simple weight shrinkage in this dataset. Elastic Net shows intermediate performance ($0.9181 \pm 0.0423$), consistent with its role in balancing sparsity and $\ell_2$ regularization.  

Our proposed methods, Covridge and Sparridge, further improve predictive performance by incorporating dual-penalty regularization. Covridge, which modifies the standard $\ell_2$  penalty via a weighted covariance transformation, achieves the highest test accuracy ($0.9718 \pm 0.0199$), indicating that accounting for feature correlations can improve generalization. This represents a substantial gain of approximately 13\% over the baseline method without regularization. Sparridge, which combines a sparsity-inducing penalty with covariance-based regularization, also attains high accuracy ($0.9673 \pm 0.0255$), confirming that combining sparsity and feature structure is effective.  

These results indicate that dual-penalty and weighted regularization provide a robust alternative to standard regularization methods. Both Covridge and Sparridge not only achieve superior accuracy but also maintain low variability, highlighting their suitability for robust high-dimensional classification. The performance of these regularized neural networks illustrates that carefully designed regularization strategies can control model complexity and enhance generalization in real-world diagnostic applications.

\section{Concluding Remarks} \label{sec:conc}

With increasing dataset sizes and feature dimensionality, neural network models involve a large number of parameters, which makes careful training crucial to avoid overfitting and to achieve reliable predictions. Neural networks are an important framework in machine learning, and their predictive performance can be substantially improved through appropriate regularization techniques. In this paper, we reviewed existing regularization methods, including Ridge, Lasso, and Elastic Net, and proposed two approaches, Covridge and Sparridge, which incorporate adaptive penalties based on feature covariance and sparsity. We evaluated the performance of all methods under different conditions of feature correlation, sparsity, and noise using simulation studies and real-life applications. The results showed that the regularized methods consistently achieve lower prediction errors and reduced bias compared to standard feedforward neural networks without regularization. The performance of Covridge and Sparridge is superior to that of other regularized methods, particularly in high-dimensional or highly correlated settings. In addition, an application to an energy dataset for building cooling load prediction further confirmed the practical utility of the proposed methods. Regularized neural networks were able to capture nonlinear relationships between building characteristics and cooling load, producing more accurate predictions than models without regularization. In addition to the regression problems, we also evaluated the proposed methods on a high-dimensional biological classification using the gene expression dataset. The classification results indicate that regularization plays a crucial role in improving predictive accuracy and reducing performance variability. 

In this work, we focused on $\ell_1$- and $\ell_2$-based regularization methods, including Ridge, Lasso, Elastic Net, and our proposed Covridge and Sparridge. These penalties are widely used and easy to interpret, which allows us to clearly evaluate the benefits of adaptive regularization. We did not include all existing regularizers, as comparing every method would have required substantially more time and computational resources. The proposed methods are flexible and can be extended in future work to incorporate other types of penalties or hybrid strategies, offering opportunities for further improvement in predictive performance. The simulation study in this work focused on regression problems; extending the proposed methods to classification tasks and evaluating them via simulation represents a natural direction for future research. However, in a real-life application to biological classification, we demonstrated the benefits of our methods on a high-dimensional classification problem using real data. Covridge and Sparridge can be computationally intensive for very high-dimensional datasets, which may require more efficient implementations. Future work could also analyze the application of these regularization methods to deeper or more complex architectures, such as convolutional or recurrent neural networks for image, temporal, or textural data.

\section*{Acknowledgments}
We acknowledge the National Academic Infrastructure for Supercomputing in Sweden (NAISS) for providing the computational resources used in this study, including access to the Alvis, Mimer, and Dardel systems.

\newpage
\bibliographystyle{plainnat} 
\bibliography{references}

\newpage
\appendix
\renewcommand{\theequation}{A\arabic{equation}}
\setcounter{equation}{0}

\renewcommand{\thetable}{A\arabic{table}}
\setcounter{table}{0}

\section*{Supplementary Materials}
\section{Additional simulation results} \label{supp:sim_results}

Table \ref{tab:simdgp4} reports the results of medium-dimensional ($p<n$) with no sparsity. All regularized methods significantly outperform the unregularized baseline. Covridge and Sparridge yield the lowest MSEs when the model is linear. Elastic Net and Ridge remain competitive, under nonlinear dependence, but Sparridge’s adaptability yields a small yet consistent advantage. Table \ref{tab:simdgp5} reports results for a high-dimensional setting with higher sparsity. It can be seen that Sparridge demonstrates the strongest performance in nonlinear models, indicating that its regularization controls both sparsity and bias for complex data structures. For linear models, Lasso performs well, particularly when feature correlation is low, while Elastic Net and Sparridge show similar performance under these conditions. Covridge performs poorly in high-dimensional linear settings with many irrelevant predictors. This limitation is because of the use of covariance weighting without an $\ell_1$ component, which prevents it from exploiting sparsity. Although tuning the shrinkage parameter or increasing optimization epochs may improve results slightly, the absence of a sparsity-inducing component fundamentally constrains its performance. In nonlinear settings, Covridge often outperforms standard unregularized networks because its covariance-adaptive shrinkage stabilizes coefficient estimates under model misspecification.


\begin{table}[H]
\centering
\caption{Performance metrics (MSE, MAE, Bias) for different methods under two noise levels (0.10 and 2.00), correlation levels ($\rho=0.25$ and $0.75$), and linear/nonlinear models when $n=1000$, $p=200$ and $k=200$.}
\resizebox{\textwidth}{!}{%
\begin{tabular}{lcccccccccccc}
\toprule
\multirow{2}{*}{Method} & \multicolumn{6}{c}{$\rho=0.25$} & \multicolumn{6}{c}{$\rho=0.75$} \\
\cmidrule(lr){2-7} \cmidrule(lr){8-13}
 & \multicolumn{2}{c}{MSE} & \multicolumn{2}{c}{MAE} & \multicolumn{2}{c}{Bias} & \multicolumn{2}{c}{MSE} & \multicolumn{2}{c}{MAE} & \multicolumn{2}{c}{Bias} \\
\cmidrule(lr){2-3} \cmidrule(lr){4-5} \cmidrule(lr){6-7}
\cmidrule(lr){8-9} \cmidrule(lr){10-11} \cmidrule(lr){12-13}
 & 0.10 & 2.00 & 0.10 & 2.00 & 0.10 & 2.00 & 0.10 & 2.00 & 0.10 & 2.00 & 0.10 & 2.00 \\
\midrule
\multicolumn{13}{l}{\textbf{Linear}} \\
\midrule
No Regularization & 61.701 & 63.971 & 6.182 & 6.361 & 0.008 & 0.009 & 35.095 & 41.093 & 4.706 & 5.187 & -0.038 & -0.094 \\
Ridge   & 27.636 & 33.795 & 3.759 & 4.414 & 0.101 & 0.001 & 24.482 & 30.621 & 3.934 & 4.456 & -0.146 & -0.212 \\
Lasso   & 23.681 & 30.601 & 3.501 & 4.208 & 0.077 & 0.064 & 22.829 & 30.361 & 3.838 & 4.445 & -0.012 & -0.095 \\
Elastic Net       & 20.261 & 27.199 & 3.062 & 3.901 & 0.032 & 0.044 & 22.687 & 30.179 & 3.781 & 4.418 & 0.013 & -0.183 \\
Covridge          & 13.150 & 21.806 & 2.438 & 3.386 & -0.034 & 0.097 & 17.634 & 29.080 & 2.972 & 4.119 & 0.090 & 0.046 \\
Sparridge          & 12.652 & 21.295 & 2.433 & 3.419 & -0.024 & -0.029 & 15.563 & 24.994 & 2.795 & 3.700 & 0.045 & 0.159 \\
\midrule
\multicolumn{13}{l}{\textbf{Nonlinear}} \\
\midrule
No Regularization & 22.267 & 25.226 & 3.679 & 3.916 & -0.232 & -0.272 & 14.759 & 19.257 & 3.004 & 3.409 & -0.157 & -0.183 \\
Ridge   & 5.522 & 11.552 & 1.231 & 2.493 & 0.006 & -0.062 & 3.065 & 10.418 & 0.934 & 2.406 & -0.103 & -0.121 \\
Lasso   & 12.053 & 12.940 & 1.764 & 2.704 & -0.061 & -0.014 & 4.251 & 11.772 & 1.218 & 2.637 & -0.138 & -0.119 \\
Elastic Net       & 5.260 & 11.299 & 1.167 & 2.473 & 0.028 & 0.026 & 3.014 & 10.213 & 0.931 & 2.386 & -0.033 & -0.092 \\
Covridge          & 5.982 & 12.608 & 1.838 & 2.770 & 0.032 & 0.015 & 3.895 & 11.367 & 1.479 & 2.632 & -0.008 & -0.067 \\
Sparridge          & 6.261 & 12.494 & 1.888 & 2.782 & -0.010 & 0.014 & 4.019 & 10.857 & 1.502 & 2.588 & -0.039 & -0.023 \\
\bottomrule
\end{tabular}%
}
\label{tab:simdgp4}
\end{table}

\begin{table}[H]
\centering
\caption{Performance metrics (MSE, MAE, Bias) for different methods under two noise levels (0.10 and 2.00), correlation levels ($\rho=0.25$ and $0.75$), and linear/nonlinear models when $n=1000$, $p=2000$ and $k=400$.}
\resizebox{\textwidth}{!}{%
\begin{tabular}{lcccccccccccc}
\toprule
\multirow{2}{*}{Method} & \multicolumn{6}{c}{$\rho=0.25$} & \multicolumn{6}{c}{$\rho=0.75$} \\
\cmidrule(lr){2-7} \cmidrule(lr){8-13}
 & \multicolumn{2}{c}{MSE} & \multicolumn{2}{c}{MAE} & \multicolumn{2}{c}{Bias} & \multicolumn{2}{c}{MSE} & \multicolumn{2}{c}{MAE} & \multicolumn{2}{c}{Bias} \\
\cmidrule(lr){2-3} \cmidrule(lr){4-5} \cmidrule(lr){6-7}
\cmidrule(lr){8-9} \cmidrule(lr){10-11} \cmidrule(lr){12-13}
 & 0.10 & 2.00 & 0.10 & 2.00 & 0.10 & 2.00 & 0.10 & 2.00 & 0.10 & 2.00 & 0.10 & 2.00 \\
\midrule
\multicolumn{13}{l}{\textbf{Linear}} \\
\midrule
No Regularization & 1320.031 & 1348.932 & 28.720 & 29.002 & -3.938 & -3.820 & 856.400 & 991.487 & 23.546 & 25.174 & -2.073 & -2.770 \\
Ridge   & 1238.485 & 1263.509 & 27.979 & 28.315 & -3.759 & -3.529 & 571.305 & 747.611 & 19.192 & 21.985 & -2.050 & -2.530 \\
Lasso   & 944.778  & 976.939  & 24.215 & 24.660 & -3.650 & -3.432 & 239.821 & 444.553 & 12.303 & 16.716 & -2.153 & -2.567 \\
Elastic Net       & 945.588  & 965.843  & 24.190 & 24.526 & -3.582 & -3.324 & 240.870 & 446.492 & 12.345 & 16.744 & -2.123 & -2.659 \\
Covridge          & 1328.511 & 1369.331 & 28.324 & 28.851 & -3.975 & -4.026 & 1406.602 & 1416.672 & 29.508 & 29.758 & -2.837 & -3.215 \\
Sparridge          & 946.686  & 963.278  & 24.192 & 24.440 & -3.962 & -3.716 & 321.510 & 512.465 & 14.282 & 17.999 & -1.507 & -1.635 \\
\midrule
\multicolumn{13}{l}{\textbf{Nonlinear}} \\
\midrule
No Regularization & 639.376 & 657.760 & 20.019 & 20.247 & -3.254 & -3.073 & 304.187 & 316.852 & 13.660 & 13.917 & -2.050 & -1.926 \\
Ridge   & 603.160 & 626.253 & 19.438 & 19.748 & -2.913 & -2.819 & 286.541 & 302.113 & 13.124 & 13.452 & -1.857 & -1.754 \\
Lasso   & 429.439 & 456.032 & 16.640 & 17.064 & -2.863 & -2.631 & 212.552 & 231.674 & 11.564 & 11.952 & -2.058 & -1.853 \\
Elastic Net       & 423.686 & 450.654 & 16.530 & 17.004 & -2.834 & -2.679 & 211.422 & 231.185 & 11.527 & 11.973 & -2.055 & -1.891 \\
Covridge          & 599.965 & 614.630 & 19.482 & 19.742 & -3.257 & -3.150 & 287.684 & 302.724 & 13.152 & 13.518 & -2.023 & -1.829 \\
Sparridge          & 382.060 & 399.945 & 15.576 & 15.869 & -2.619 & -2.513 & 183.089 & 202.216 & 10.605 & 11.092 & -1.539 & -1.471 \\
\bottomrule
\end{tabular}%
}
\label{tab:simdgp5}
\end{table}


\end{document}

%% file: preamble.tex
\usepackage{setspace}
\usepackage{lipsum}
\usepackage{amsmath}
\usepackage{amssymb}
\usepackage{amsthm}
\usepackage{forloop}
\usepackage{nicefrac}
\usepackage{mathtools}
\usepackage[table]{xcolor}
\usepackage{url}
\usepackage{bm}
\usepackage{float}
\usepackage{hyperref}
\hypersetup{
    colorlinks,
    linkcolor={blue!100!black!70},
    citecolor={blue!75!black},
    urlcolor={blue!75!black}
}
\usepackage{makecell}
\usepackage[normalem]{ulem}
\usepackage[nameinlink,capitalise]{cleveref}
\usepackage[shortlabels]{enumitem}
\usepackage{tikz}
\usetikzlibrary{positioning}
\usetikzlibrary{arrows.meta}
\usepackage{cancel}
\usepackage{multirow}
\usepackage{makecell}
\usepackage{caption}
\usepackage{subcaption}
\usepackage{framed}
\usepackage{wrapfig}
\usepackage{booktabs}
\usepackage{mathrsfs}
\usepackage[linesnumbered, ruled]{algorithm2e}
\usepackage[round,sort&compress]{natbib}
\usepackage[splitrule]{footmisc}
\usepackage{booktabs}
\usepackage{amsfonts}
\usepackage{amsmath, amssymb, amsthm, amsfonts, mathtools, bbm, bm, graphics, caption, float, latexsym, titlesec, enumitem, comment}
\usepackage{amssymb} 
\usepackage{tikz}
\usetikzlibrary{positioning}
\usepackage{booktabs}
\usepackage{siunitx}
\usepackage{caption}
\usepackage{adjustbox}
\usepackage{varwidth}
\usepackage{graphicx}
\usepackage{subcaption}
\usepackage{tabularx}


\newtheorem{theorem}{Theorem}[section]

\newtheorem{remark}{Remark}[section]

\crefname{definition}{Def.}{Defs.}
\crefname{equation}{Eq.}{Eqs.}

\newlist{assump}{enumerate}{1}
\setlist[assump]{label=(A\arabic*), ref=A\arabic*, resume}

\newcommand{\defcal} [1]{\expandafter\newcommand\csname cal#1\endcsname{{\mathcal #1}}}
\newcommand{\defsf} [1]{\expandafter\newcommand\csname sf#1\endcsname{{\mathsf #1}}}
\newcommand{\defbf} [1]{\expandafter\newcommand\csname bf#1\endcsname{{\mathbf #1}}}
\newcommand{\defbb} [1]{\expandafter\newcommand\csname bb#1\endcsname{{\mathbb{#1}}}}
\newcommand{\deffrak} [1]{\expandafter\newcommand\csname frak#1\endcsname{{\mathfrak{#1}}}}

\newcounter{ct}
\forLoop{1}{26}{ct}{
    \edef\letter{\Alph{ct}}
    \expandafter\defcal\letter
    \expandafter\defsf\letter
    \expandafter\defbf\letter
    \expandafter\defbb\letter
    \expandafter\deffrak\letter
}
\forLoop{1}{26}{ct}{
    \edef\letter{\alph{ct}}
    \expandafter\defcal\letter
    \expandafter\defsf\letter
    \expandafter\defbf\letter
    \expandafter\defbb\letter
    \expandafter\deffrak\letter
}

\newcommand{\mirrorfunc}[1]{
    \ensuremath{\if$#1$\phi \else \phi(#1)\fi}
}
\newcommand{\mirrorfuncdual}[1]{
    \ensuremath{\if$#1$\phi^{*}\else \phi^{*}(#1)\fi}
}

\newcommand{\subscript}[2]{$#1 _ #2$}
\newlist{assumplist}{enumerate}{1}
\setlist[assumplist]{label=\subscript{\textbf{\textsf{A}}}{\textsf{{\arabic*}}},leftmargin=*, itemsep=0pt}

\newcommand{\potential}[1]{
    \ensuremath{\if$#1$f\else f(#1)\fi}
}
\newcommand{\metric}[1]{
    \ensuremath{\if$#1$\mathscr{G}\else \mathscr{G}(#1)\fi}
}

\makeatletter
\newcommand*{\algotitle}[2]{%
  \stepcounter{algocf}%
  \hypertarget{algocf.title.\theHalgocf}{}%
  \NR@gettitle{#1}%
  \label{#2}%
  \addtocounter{algocf}{-1}%
}
\makeatother